%% file: main-arxiv.tex
\documentclass[11pt]{article}
\usepackage[margin=1in]{geometry}

\usepackage[utf8]{inputenc} % allow utf-8 input
\usepackage[T1]{fontenc}    % use 8-bit T1 fonts
\usepackage{hyperref}       % hyperlinks
\usepackage{url}            % simple URL typesetting
\usepackage{booktabs}       % professional-quality tables
\usepackage{amsfonts}       % blackboard math symbols
\usepackage{nicefrac}       % compact symbols for 1/2, etc.
\usepackage{microtype}      % microtypography
\usepackage{xcolor}         % colors

\title{Optimism in the Face of Ambiguity Principle \\for Multi-Armed Bandits}

\author{
Mengmeng Li\thanks{Risk Analytics and Optimization Chair, EPFL \href{mailto:mengmeng.li@epfl.ch}{\texttt{mengmeng.li@epfl.ch}}} \and 
Daniel Kuhn\thanks{Risk Analytics and Optimization Chair, EPFL \href{mailto:daniel.kuhn@epfl.ch}{\texttt{daniel.kuhn@epfl.ch}}} \and 
Bahar Ta{\c s}kesen\thanks{Booth School of Business, University of Chicago \href{mailto:bahar.taskesen@chicagobooth.edu}{\texttt{bahar.taskesen@chicagobooth.edu}}} 
}

\input{style}

\begin{document}
\maketitle

\begin{abstract}
Follow-The-Regularized-Leader (FTRL) algorithms often enjoy optimal regret for adversarial as well as stochastic bandit problems and allow for a streamlined analysis. However, FTRL algorithms require the solution of an optimization problem in every iteration and are thus computationally challenging. 
In contrast, Follow-The-Perturbed-Leader (FTPL) algorithms achieve computational efficiency by perturbing the estimates of the rewards of the arms, but their regret analysis is cumbersome. We propose a new FTPL algorithm that generates optimal policies for both adversarial and stochastic multi-armed bandits. Similar to FTRL, our algorithm admits a unified regret analysis, and similar to FTPL, it offers low computational costs. Unlike existing FTPL algorithms that rely on independent additive disturbances governed by a \textit{known} distribution, we allow for disturbances governed by an \textit{ambiguous} distribution that is only known to belong to a given set and propose a principle of optimism in the face of ambiguity. Consequently, our framework generalizes existing FTPL algorithms. It also encapsulates a broad range of FTRL methods as special cases, including several optimal ones, which appears to be impossible with current FTPL methods. Finally, we use techniques from discrete choice theory to devise an efficient bisection algorithm for computing the optimistic arm-sampling probabilities. This algorithm is up to $10^4$ times faster than standard FTRL algorithms that solve an optimization problem in every iteration. Our results not only settle existing conjectures but also provide new insights into the impact of perturbations by mapping FTRL to FTPL.
\end{abstract}

%%%%%%%%%%%%%%%%%%%%%%%%%%%%%%%%%%%%%%%%%%%%%%%%%%%%%%%%%%%%%%%%%%%%%%

% Text of your paper here
\section{Introduction}

We consider multi-armed bandit problems in which a learner interacts with an environment over~$T$ rounds. In each round, the learner selects one of~$K$ arms and then observes and receives an uncertain reward associated with the chosen arm. The learner's objective is to minimize regret, which we define as the absolute difference between the total expected reward obtained and the total expected reward that could have been achieved with perfect knowledge of the reward distribution. In the stochastic setting, where the rewards in each round are drawn independently from an unknown but fixed distribution, the {\em Upper Confidence Bound} algorithm \citep{auer2002finite} as well as the {\em Thompson Sampling} algorithm \citep{thompson1933likelihood} achieve the optimal $\mathcal{O}(\log T)$ regret \citep{bubeck2012regret}. In the adversarial setting, where rewards are strategically chosen by a malicious adversary, however, these methods suffer from linear regret \citep{zimmert2021tsallis}. In contrast, the Follow-the-Regularized-Leader (FTRL) algorithm by \citet{gordon1999regret}, which uses the iterates of a gradient descent-type algorithm as arm-sampling distributions, often achieves the optimal $\mathcal{O}(\sqrt{KT})$ regret in the adversarial setting \citep{bubeck2012regret}.

Prior knowledge about the nature of the environment is typically unavailable. Therefore, an algorithm that achieves optimal regret in both stochastic and adversarial settings simultaneously is highly desirable. Recently, \cite{zimmert2021tsallis} proved that an FTRL algorithm with a Tsallis entropy regularizer can \textit{simultaneously} achieve the optimal $\mathcal{O}(\log T)$ regret in the stochastic setting as well as the optimal $\mathcal{O}(\sqrt{KT})$ regret in the adversarial setting, without requiring parameter tuning. Algorithms of this type are often said to exhibit the ``best-of-both-worlds'' (BOBW) property \citep{bubeck2012best}. The results by~\citet{zimmert2021tsallis} have been extended in various directions, aiming to identify the key properties of regularizers that induce the BOBW property~\citep{jin2024improved}. However, FTRL algorithms require solving a computationally expensive optimization problem in each round to determine the arm-sampling distribution. 

Follow-the-Perturbed-Leader (FTPL) algorithms \citep{hannan1957approximation} select an arm with a maximal perturbed reward estimate, where the perturbation is sampled from a prescribed noise distribution. These algorithms are widely favored for their superior computational efficiency compared to FTRL approaches \citep{abernethy2014online,lattimore2020bandit}. Recently, it was shown that FTPL with Fr\'echet perturbations possesses the BOBW property \citep{ref:honda23follow}. However, this analysis heavily relies on the specifics of a Fréchet distribution with a particular shape parameter. While this paper was under review, a more systematic analysis of FTPL algorithms was provided by~\cite{lee2024follow}, further highlighting the effectiveness of FTPL methods. %Nonetheless, this analysis still depends significantly on extreme value theory and does not share substantial commonalities with the FTRL-based approach.

It is well known that any FTPL policy can be expressed as an FTRL policy~\citep{abernethy2016perturbation,hofbauer2002global}. However, the reverse does not hold in general \cite[Proposition2.2]{hofbauer2002global}. Establishing a one-to-one correspondence between meaningful subclasses of FTPL and FTRL policies remains an open problem \citep{abernethy2016perturbation}. Although all FTRL and FTPL methods can be viewed as instances of a Gradient-Based Prediction Algorithm (GBPA)\citep{ref:abernethy2015fighting}, their regret analyses require separate techniques. Furthermore, an open question posed by \citet{ref:kim2019optimality} asks whether there exists a noise distribution such that the corresponding FTPL policy exactly matches the FTRL policy with the Tsallis entropy regularizer. \citet{ref:kim2019optimality} also proved that FTPL with independent and identically distributed (i.i.d.) noise across the arms cannot recover Tsallis-entropy-regularized FTRL. Designing perturbations that precisely replicate the FTRL algorithm with BOBW capability is essential for understanding the role of regularization through perturbation. Resolving this open question would also facilitate the unification of FTRL and FTPL regret analysis.

\paragraph{Contributions.} In this paper, we bridge the gap between FTRL and FTPL methods by studying \textit{ambiguous} noise distributions that allow for \textit{correlations} across the arms. Additionally, we introduce a new ``optimism in the face of ambiguity'' principle, whereby the perturbations in FTPL are sampled from the most advantageous noise distribution within a prescribed ambiguity set. This contrasts sharply with standard FTPL algorithms, which rely on a single fixed noise distribution. We derive explicit formulas for this most advantageous noise distribution, thus resolving the open problem posed by \citet{ref:kim2019optimality}. Leveraging techniques from discrete choice theory~\citep{natarajan2009persistency,feng2017relation}—traditionally studied in economics and psychology—we show that the arm-sampling probabilities under the optimal noise distribution can be computed highly efficiently using bisection. Unlike standard FTRL algorithms, which require solving an expensive optimization problem in every round, our approach is significantly more computationally efficient and its runtime remains comparable to that of FTPL, up to logarithmic factors. As a result, our algorithm combines the unified regret analysis of FTRL with the computational efficiency of FTPL. Moreover, it encompasses a broad class of FTRL methods as special cases, including several optimal ones, such as those based on Tsallis entropy and hybrid regularizers. Notably, while unifying these FTRL methods within the traditional FTPL framework was previously considered infeasible, our approach successfully achieves this integration.

\paragraph{Related work.} 
In this paper, we relax the assumption of i.i.d.\ arm perturbations, thereby generalizing traditional FTPL methods. The i.i.d.\ assumption, which underlies most FTPL algorithms, is also relaxed in~\citep{melo2023discrete} by interpreting the arm-sampling probabilities as choice probabilities in a nested logit model, a concept commonly studied in discrete choice theory. In this work, however, the noise must follow a generalized extreme-value distribution, and the resulting algorithm does not achieve BOBW regret bounds. In contrast, we work with a whole family of distributions and leverage ideas from discrete choice theory and distributionally robust optimization to develop an efficient bisection algorithm for computing arm-sampling probabilities under the most advantageous noise distribution. This general framework encompasses several algorithms that achieve BOBW regret bounds. The FTPL algorithm with i.i.d.\ Fréchet-distributed noise is also known to exhibit BOBW capabilities \citep{ref:honda23follow}, but its regret analysis is specifically tailored to the Fréchet distribution. While this method extends to some other noise distributions \citep{lee2024follow}, the underlying regret analysis remains complex. In contrast, our perturbation-based algorithm achieves BOBW regret bounds by leveraging its exact equivalence with FTRL algorithms that possess the BOBW property, leading to a more unified and efficient approach.

\paragraph{Notation.}
We denote by $[K] = \{1, \ldots, K\}$ the set of all integers up to $K \in \mathbb N$.
The probability simplex over $[K]$ is defined as $\Delta^K=\{\bs p\in\mb R_+^{K}: \sum_{k=1}^K p_k=1\}$. We use $\bs e_i$ with $i\in[K]$ to denote the $i$-th standard basis vector of~$\mb R^K$. The Bregman divergence function induced by a differentiable function~$\phi: \R^d \to \R$ is defined through~$\mathds D_{\phi}(\bs x, \bs y)=\phi(\bs x)-\phi(\bs y)-\langle \bs x-\bs y,\nabla \phi(\bs y)\rangle.$

\section{Multi-Armed Bandits} \label{section:MAB}

We study a multi-armed bandit (MAB) problem running over $T \in \mathbb{N}$ rounds. In each round the learner must select one of $K\in\mb N$ arms and earns a random reward that depends on the chosen arm. More precisely, in round $t \in [T]$, the learner selects an arm $a_t \in [K]$, the environment generates a reward vector $\bs{r}_t = (r_{t,1}, r_{t,2}, \ldots, r_{t,K}) \in [-1, 0]^K$, and the learner receives the reward $r_{t,a_t}$ associated with arm~$a_t$. The crux of MAB problems is that the learner observes only the reward~$r_{t, a_t}$ associated with the chosen arm but receives no information about the rewards~$r_{t, k}$ of the other arms~$k \neq a_t$. Accordingly, we assume throughout the paper that the arm~$a_t$ is sampled from a distribution over~$[K]$ chosen by the learner that may depend on the history $(a_1,\ldots, a_{t-1})$ of the chosen arms and the history $(r_{1,a_1},\ldots, r_{t-1,a_{t-1}})$ of the corresponding rewards. Similarly, the reward vector~$\bs r_t$ is sampled from a distribution over~$[-1,0]^K$ unknown to the learner. We distinguish two main reward generation regimes. In the (non-oblivious) adversarial regime, the reward distribution may depend on the history $(a_1,\ldots, a_{t-1})$ of the chosen arms as well as the history $(\bs r_{1},\ldots, \bs r_{t-1})$ of the rewards. In the stochastic regime, on the other hand, the reward distribution is kept fixed, and the rewards are sampled independently from this distribution. There are also intermediate reward generation regimes under which the environment has varying levels of adversarial power.

%{\color{blue}from some distribution $\bs p_t(\cdot)=\mb P(\cdot|a_1,\ldots,a_{t-1},\bs{r}_{1,a_1},\ldots,\bs{r}_{t-1,a_{t-1}})$ that depends on the history of actions and \textit{observed} rewards.} {\color{blue}from an arbitrary distribution $\mb P_t(\cdot)=\mb P(\cdot|a_1,\ldots,a_{t-1},\bs{r}_1,\ldots,\bs{r}_{t-1})$ that depends on the history of actions and rewards.} 
% The learner observes and receives the reward \( r_{t, a_t} \) associated with the chosen arm, but receives no information about the values~$r_{t, k}$ for $k \neq a_t$ of the other arms. 
% Throughout the paper, we usually work with the bandit feedback setting, where t

The learner's objective is to minimize the 
% pseudo-
regret
% \footnote{For pseudo-regret, we compare the best expected reward with the collected reward in expectation. This is in contrast to the random regret, which compares the best expected reward with the collected random reward.} 
\[
\mathcal{R}(T) = \max_{k \in [K]} \mathbb{E}\left[\sum_{t=1}^T r_{t,k}\right] - \mathbb{E}\left[\sum_{t=1}^T r_{t, a_t}\right],
\]
which measures the difference between the expected cumulative reward of the best arm under full distributional information and the learner's expected cumulative reward. Here, the expectations are taken with respect to the arm-sampling distributions chosen by the learner and the reward distributions chosen by the environment. %internal randomness of the algorithm used for selecting the arms and the external randomness of the environment used for generating the rewards. 
We highlight that $\mathcal R(T)$ is sometimes termed {\em pseudo-regret} \citep{zimmert2021tsallis}. As we do not distinguish different notions of regret in this paper, we simply refer to $\mathcal R(T)$ as the {\em regret} to keep terminology simple. We also emphasize that~$\mathcal R(T)$ may be negative for certain arm-sampling and reward distributions. Indeed, if the reward distribution changes over time, it may be strictly suboptimal to select the same arm in each round. However, the {\em worst-case} regret, which is obtained by maximizing $\mathcal{R}(T)$ over all admissible reward distributions, is nonnegative even in the stochastic regime, where the environment has only minimal adversarial power. Indeed, if the reward of each arm follows a fixed Bernoulli distribution independent of~$t$, then the regret is already lower bounded by $\mac O(\sqrt{KT})$~\citep{auer1995gambling}.

\begin{algorithm}[H]
\caption{Gradient-based prediction algorithm (GBPA)}
\begin{algorithmic}
\REQUIRE Differentiable convex function~$\phi$ with $\nabla_{\bs u} \phi(\bs u) \in \Delta^K$ 
\STATE Initialize $\hat{\bs u}_0 = \bs 0 $ 
\FOR{round $t = 1, \ldots, T$}
    \STATE Environment chooses a reward vector~$\bs r_t \in [-1,0]^K$
    \STATE \label{line:sampling} Learner chooses $a_t \sim \bs  p_t = \nabla_{\bs u} \phi(\bs u)|_{\bs u = \hat{\bs u}_{t-1}}$ and receives reward~$r_{t, a_t}$
    \STATE Learner estimates single-round reward vector~$\hat{\bs r}_t = (r_{t, a_t} / p_{t, a_t}) \bs e_{a_t}$
    \STATE Update $\hat{\bs u}_t \gets \hat{\bs u}_{t-1} + \hat{\bs r}_t$
    \ENDFOR
\end{algorithmic}
\label{alg:mab-gbpa}
\end{algorithm}

In this paper, we use different variants of a gradient-based prediction algorithm (GBPA) \citep{ref:abernethy2012interior,abernethy2014online, ref:abernethy2015fighting, ref:kim2019optimality} to select the arms; see Algorithm~\ref{alg:mab-gbpa}. GBPA recursively constructs a statistic~$\hat{\bs u}_{t-1}\in\mathbb R^K$ whose $k$-th component estimates the expected cumulative reward achievable by pulling arm~$k$ in each of the rounds $1,\ldots,t-1$. In round~$t \in [T]$, GBPA uses the gradient of a convex potential function~$\phi : \R^K \to \R$ evaluated at~$\hat{\bs u}_{t-1}$ as an arm-sampling distribution $\bs p_t$ and samples an arm~$a_t$ from~$\bs p_t$. Next, GBPA updates~$\hat{\bs u}_{t-1}$ by adding the single-round reward estimate~$\hat{\bs r}_t = (r_{t, a_t} / p_{t, a_t}) \bs e_{a_t}$. One readily verifies that if~$\bs p_t> \bs 0$, then $\hat {\bs r}_t$ constitutes an unbiased estimator for $\mathbb E[\bs r_t]$. GBPA unifies several MAB algorithms, including those described in~\citep{ref:auer2002nonstochastic, ref:kujala2005following, ref:neu2013efficient}. Additionally, it encompasses various follow-the-leader-type algorithms widely used in sequential decision-making with full information (where the learner observes the full reward vector~$\bs{r}_t$ and not only the reward of the chosen arm). These algorithms differ primarily in the choice of the convex potential function~$\phi$ used as an input to GBPA. Below we discuss the policies 
corresponding to different choices of~$\phi.$

\paragraph{Follow-the-leader (FTL).} GBPA with $\phi(\bs u) =\max_{\bs p \in \Delta^K} \bs p^\top \bs u $ is known as the FTL algorithm. In this case we have $\nabla_{\bs u} \phi(\bs u) \in \argmax_{\bs{p}\in\Delta^K} \bs p^\top \bs u$ by Danskin's theorem~\citep[Proposition~B.25]{bertsekas2016nonlinear}, that is, the learner simply chooses an arm with maximal cumulative reward estimate.\footnote{In this informal discussion we disregard technical complications arising when the maximizer~$\bs p$ is not unique for a given reward estimate~$\bs u$, in which case the potential function~$\phi(\bs u)$ fails to be differentiable at~$\bs u$.}

\noindent While FTL is easy to implement, it is well known that the regret of FTL can grow linearly with~$T$ even if there are only $K=2$ arms. For example, if the adversary chooses the reward vectors~$\bs r_{1}=\{-1/2,0\}$, $\bs r_t = \{-1,0\}$ when $t>1$ is odd and $\bs r_{t}=\{0,-1\} $ when $t$ is even, then one can show that FTL selects arm~$1$ whenever~$t$ is odd and arm~$2$ whenever~$t$ is even. Thus, the cumulative reward at time~$T$ is at most~$(-T+1)/2$, and the regret is at least $T/2-1$~\cite[Chapter~5]{ref:hazan2016introduction}.

\noindent \paragraph{Follow-the-regularized-leader (FTRL).} A popular approach to stabilize the FTL algorithm is to add a convex regularization function~$\psi : \mathbb R^K \to \R$ to the linear objective function~$\bs p^\top \bs u$. In this case, the learner generically constructs a non-degenerate arm-sampling distribution by solving $\max_{\bs p \in \Delta^K} \bs p^\top \bs u - \psi(\bs p)$. GBPA with potential function $\Phi^R(\bs u; \psi) = \max_{\bs p \in \Delta^K} \bs p^\top \bs u - \psi(\bs p)$ is known as the FTRL algorithm. 
In the adversarial regime, the FTRL algorithm achieves the minimax optimal regret of~$\mathcal{O}(\sqrt{KT})$ if $\psi(\bs p)=\eta\psi_\alpha^{\mathds T}(\bs p),$ where $\eta = \sqrt{T(1-\alpha)/ (2\alpha)}$ is the learning rate and
\begin{equation}\psi_\alpha^{\mathds T}(\bs p) = \frac{1-\sum_{k=1}^K p_k^\alpha}{1-\alpha} \quad\forall \bs p\in \R^K\vspace{0pt}
\label{eq:tsallis-entropy}
\end{equation}
is the Tsallis entropy 
with parameter~$\alpha \in (0,1)$ \cite[Corollary~3.2]{ref:abernethy2015fighting}.
Similarly, in the stochastic regime, the FTRL algorithm achieves the optimal regret of $\mathcal{O}(\log T)$ if the potential function $\psi(\bs p)=\eta_t\psi_\alpha^{\mathds T}(\bs p)$ scales with a time-dependent learning rate $\eta_t=2\sqrt{t}$ \cite[Theorem~2]{ito2021parameter}.
If FTRL is equipped with a hybrid regularizer that combines the Shannon entropy with the Tsallis entropy, then it enjoys a BOBW capability, that is, it can be shown to achieve optimal regret both in the adversarial as well as in the stochastic bandit regime \citep{zimmert2019beating}. Thus, the FTRL algorithm comes with strong statistical guarantees. On the flipside, however, it is computationally expensive because it requires the solution of a different convex optimization problem in each round.

\paragraph{Follow-the-perturbed-leader (FTPL).} As an alternative to the computationally expensive FTRL method, it has been proposed to inject stochastic noise~$\bs z \sim \QQ\in\mathcal P(\R^K)$ into the cumulative reward estimate~$\bs u$ and to sample arms from $\bs p= \EE_{\bs z \sim \QQ}[\bs e_{k\opt(\bs z)}]$, where $k\opt(\bs z) \in \argmax_{k \in [K]} u_k + z_{k}$. Using the dominated convergence theorem in conjunction with Danskin's theorem \cite[Proposition~B.25]{bertsekas2016nonlinear}, one can show that the arm-sampling distribution~$\bs p$ coincides with the gradient of the potential function $\Phi^P({\bs u}; \QQ) = \EE_{\bs z \sim \QQ} [\max_{\bs p \in \Delta^K} \bs p^\top ({\bs u} + \bs z)]$. GBPA with $\phi(\bs u)=\Phi^P({\bs u}; \QQ)$ is known as the FTPL algorithm. Existing FTPL algorithms assume that the disturbances associated with different arms (that is, the components of~$\bs z$) are mutually independent under~$\QQ$.

FTPL algorithms are computationally efficient because they simply select the arm with the maximum perturbed reward and because this arm can be identified by searching. This is significantly cheaper than solving a convex optimization problem. However, due to their stochastic nature, the analysis of FTPL algorithms is more cumbersome compared to the straightforward and mature analysis of FTRL algorithms. Even though FTPL algorithms have been shown to enjoy a BOBW capability~\citep{ ref:honda23follow, lee2024follow}, it is unclear whether there exists an algorithm that is as efficient as an FTPL method yet admits a streamlined analysis like an FTRL algorithm. 

%{\color{blue}Note that  However, the reverse direction generally does not hold~\cite[Proposition~2.2]{hofbauer2002global}. Determining a one-to-one correspondence between relevant subclasses of FTPL and FTRL policies remains an open problem~\citep{abernethy2016perturbation}.}

A promising approach to achieving this goal is to establish a correspondence between FTRL and FTPL algorithms. It is well known that essentially any FTPL algorithm can be represented as an FTRL algorithm \citep[Theorem~2.1]{hofbauer2002global}. The reverse problem of framing a given FTRL algorithm as an FTPL algorithm, however, is %to find a disturbance distribution~$\QQ$ such that the corresponding FTPL algorithm uses the same arm-sampling distributions as the FTRL algorithm induced by some given convex regularization function~$\psi$. This task is generally
perceived as challenging \citep{abernethy2016perturbation, ref:honda23follow}. 
%As pointed out by \cite{ref:kim2019optimality}, 
Accordingly, finding a bridge between regularization- and perturbation-based algorithms constitutes indeed an unresolved open problem.\\[2ex]
\vspace{5pt}
\noindent\fbox{%
\parbox{0.98\textwidth}{%
\noindent\textbf{Open Problem:} 
\textit{Given a convex regularization function $\psi:\mathbb{R}^K \to \mathbb{R}$ and a reward estimate $\bs u$, 
construct a perturbation distribution~$\QQ$ on~$\mathbb{R}^K$ such that $\nabla_{\bs{u}} \Phi^P(\bs{u}; \QQ) = \nabla_{\bs{u}} \Phi^R(\bs{u}; \psi)$.}
}
}\vspace{5pt}

Since FTRL with Tsallis entropy regularizer achieves the minimax optimal regret rate in adversarial bandits, a simpler but still interesting version of the above open problem is to seek an FTPL algorithm with the same arm-sampling distribution as the special instance of the FTRL algorithm with Tsallis entropy regularizer. But to date, even for this simpler problem, only a negative result is available. \citet[Theorem~8]{ref:kim2019optimality} show that FTRL with Tsallis entropy regularizer cannot be recovered by any FTPL algorithm with mutually independent disturbances~$z_k$, $k\in[K]$. Another negative result is due to \citet[Proposition~2.2]{hofbauer2002global}, who identify generalized FTRL algorithms that correspond to {\em extended real-valued} regularization functions and that cannot be matched by any FTPL algorithm. However, we assume here that $\psi$ is real-valued.

In the next section, we describe a general framework for 
mapping regularization functions in FTRL to disturbance distributions in FTPL, thus providing a systematic solution to the open problem mentioned above. To circumvent the impossibility result by \cite{ref:kim2019optimality}, we will study
\textit{ambiguous} noise-sampling distributions that allow for \textit{correlations} across the arms.

\section{Distributionally Optimistic Perturbations}

We now introduce a new class of smooth potential functions that can be viewed as best-case expected utilities of the type studied in semi-parametric discrete choice theory. That is, we define
\begin{equation}
    \Phi( \bs u; \mathcal{B}) = \sup\limits_{\QQ \in \mathcal{B}} \EE_{\bs z \sim \QQ} \left[ \max\limits_{k \in [K]} (u_k+  z_k) \right],
    \label{eq:discrete-best-case}
\end{equation}
where $ \bs z $ represents a random vector of perturbations governed by a distribution $\QQ$ from within some \textit{ambiguity set}~$\mathcal{B} \subseteq \mathcal{P}(\R^K)$. Note that if $\mathcal{B}$ is a singleton that contains only the Dirac distribution at the origin of~$\R^K$, then Algorithm~\ref{alg:mab-gbpa} with potential function~$\phi(\bs u)=\Phi(\bs u;\mathcal{B})$ reduces to FTL. In addition, $\Phi(\bs u; \{\QQ\})$ trivially coincides with $\Phi^P(\bs u; \QQ)$. Hence, GBPA with potential function~$\phi(\bs u)=\Phi(\bs u;\mathcal{B})$ generalizes traditional FTPL, which injects i.i.d.\ noise into the cumulative reward estimates.

The family of GBPA algorithms with potential function~$\phi(\bs u)=\Phi(\bs u;\mathcal{B})$ also includes the Exp3 algorithm by~\cite{auer1995gambling}, which is arguably one of the most popular FTPL algorithms.

\begin{remark}[Exp3 algorithm]\label{remark:exp3:iid}
If~$\mathcal{B}=\{\QQ\}$ is a singleton with $\QQ=\otimes_{k=1}^K \QQ_k$ and if~$\QQ_k\in\mathcal{P}(\R)$ is a Gumbel distribution with zero mean and variance $\pi^2\eta^2 / 6$ for some $\eta>0$, then one can show that $\Phi(\bs u;\mathcal{B})=\eta\log (\sum_{k=1}^K \exp (  u_k/\eta))$. In this case, the arm-sampling probabilities are available in closed form and are equivalent to the choice probabilities in the celebrated multinomial logit model in discrete choice theory, that is, $ p_k(\bs u) =(\nabla_{\bs u}\Phi(\bs u,\mathcal{B}))_k=\exp( u_k/\eta)/(\sum_{j=1}^K \exp( u_j/\eta))$, see \cite[Theorem~5.2]{mcfadden1981econometric}. This reveals that GBPA with potential function $\phi(\bs u)=\Phi(\bs u;\mathcal{B})$ reduces indeed to the celebrated Exp3 algorithm by \citet{auer1995gambling}.
\end{remark}

From now on we focus on \textit{marginal} ambiguity sets, which specify the marginal distributions of the components of~$\bs z$ but do not impose any constraints on their dependence structure. 

\begin{definition}[Marginal ambiguity set]
The marginal ambiguity set induced by~$K$ cumulative distribution functions~$F_k:\R\to[0,1]$, $k\in[K]$, is given by
 \begin{equation}\label{eq:marginal:set}
     \begin{aligned}
        \mathcal{B} = \big\{\QQ \in \mathcal{P}(\R^K) : \;\QQ[z_k \leq s ] = F_k(s)  \;\; \forall s \in \R,\;\forall k \in [K] \big\}.
 \end{aligned}
    \end{equation}
\end{definition}

We henceforth refer to GBPA with potential function~$\phi(\bs u)=\Phi(\bs u;\mathcal{B})$ induced by a marginal ambiguity set~$\mac B$ as the {\em distributionally optimistic perturbation algorithm} (DOPA). DOPA establishes a bridge between many commonly used FTRL and FTPL methods. To see this, we first recall an important property of marginal ambiguity sets, which was initially discovered in the context of semi-parametric discrete choice theory. Below we denote by $F_k^{-1}:[0,1] \rightarrow \mathbb{R}$ the (left) quantile function corresponding to the cumulative distribution function $F_k$. It is defined via
$$
F_k^{-1}(s)=\inf \left\{t: F_k(t) \geq s\right\} \quad \forall s \in \mathbb{R}.
$$

\begin{lemma}{{\cite[Theorem~1]{natarajan2009persistency}}}\label{lemma:distributional-regularization}
If $\mathcal{B}$ is a marginal ambiguity set of the form~\eqref{eq:marginal:set} and if the cumulative distribution functions $F_k, k \in[K]$, are continuous and strictly increasing in~$s$ whenever $F_k(s)\in(0,1)$, then the potential function~\eqref{eq:discrete-best-case} is convex and differentiable in $ \bs u$ and satisfies
\begin{equation}
\label{eq:frechet-reg-max}
    \Phi(\bs u; \mathcal{B})= \max _{\bs p \in \Delta^K} \sum_{k=1}^K  u_k p_k + \sum_{k=1}^K \int_{1-p_k}^1 F_k^{-1}(t) \diff t  \quad \forall \bs u\in\R^K.
\end{equation}
In addition, the unique maximizer of the convex program~\eqref{eq:frechet-reg-max} is given by $\bs p(\bs u) = \nabla_{\bs u}\Phi(\bs u; \mathcal{B})$.
\end{lemma}
Lemma~\ref{lemma:distributional-regularization} reveals that if~$\mac B$ is any marginal ambiguity set, then~$\Phi(\bs u;\mathcal{B})$ can be expressed as the optimal value of a convex maximization problem over the probability simplex. Besides its relevance for semi-parametric discrete choice theory, Lemma~\ref{lemma:distributional-regularization} also has interesting ramifications for semi-discrete optimal transport \citep[Proposition~3.6]{ref:taskesen2023semi}. Note that the objective function of the convex program in~\eqref{eq:frechet-reg-max} represents a sum of $K$ strictly concave and differentiable functions $\varphi_k(p_k)=u_k p_k+\int_{1-p_k}^1 F_k^{-1}(t)\diff t$ on $(0,1)$, $k\in[K]$, provided that~$F_k(s)$ is strictly increasing in~$s$ whenever $F_k(s)\in(0,1)$. Indeed, the derivative of $\varphi_k$ satisfies $\varphi'_k(p_k)=u_k+F_k^{-1}(1-p_k)$, which is strictly decreasing in $p_k$ because $F_k(s)$ is strictly increasing in~$s$ and~$1-p_k$ is strictly decreasing in~$p_k$. Moreover, if $F_k(s)$ is strictly increasing at {\em every}~$s\in\R$, then $\lim_{p_k\to 0}\varphi'_k(p_k)=\infty$. In this case, the optimal arm-sampling distribution $\bs p(\bs u) = \nabla_{\bs u}\Phi(\bs u; \mathcal{B})$ that solves~\eqref{eq:frechet-reg-max} satisfies $\bs p(\bs u)>\bs 0$. 

Lemma~\ref{lemma:distributional-regularization} implies that DOPA can be viewed as an FTRL algorithm with convex regularization function $\psi (\bs p)=-\sum_{k=1}^K \int_{1-p_k}^1 F_k^{-1}(t) \diff t $. Conversely, the following proposition shows that FTRL algorithms with separable regularization functions can also be interpreted as instances of DOPA.

\begin{proposition}[FTRL vs.\ DOPA]\label{prop:dopa-ftrl}
Define~$\psi:[0,1]^K\to\R$ through~$\psi(\bs p)=\sum_{k=1}^K \psi_k(p_k)$, where~$\psi_k:[0,1]\to\R$ is strictly convex and differentiable for every $k\in[K]$. If~$\mac B$ is a marginal ambiguity set of the form~\eqref{eq:marginal:set} induced by cumulative  distribution functions~$F_k:\R\to[0,1]$ satisfying $F_k(s) = \min\{1, \max\{0, -(\psi_k')^{-1}(1-s)\}\}$ for all $s\in\R$ and $k \in[K],$
then $\nabla_{\bs u}\Phi(\bs u ; \mac B) = \nabla_{\bs u}\Phi^R(\bs u; \psi)$.
\end{proposition}
\begin{proof} 
We may assume without loss of generality that $\psi_k(0)=0$. Otherwise, we can simply shift $\psi_k$ by $-\psi_k(0)$ without affecting $\nabla_{\bs u}\Phi^R(\bs u; \psi)$. Note also that $\psi_k'$ is strictly increasing because $\psi_k$ is strictly convex. Hence, the function $G_k:[0,1]\to\R$ with $G_k(s)=-\psi_k'(1-s)$ is also strictly increasing. The fundamental theorem of calculus thus implies that $-\int_{1-p_k}^1 G_k(t) \diff t=\psi_k(p_k)$. 
% Note first that $\psi_k'$ is strictly increasing because $\psi_k$ is strictly convex. Hence, the function $G_k:[0,1]\to\R$ defined through $G_k(s)=-\psi_k'(1-s)$ is also strictly increasing. As $\psi_k(0)=0,$ the fundamental theorem of calculus thus implies that $-\int_{1-p_k}^1 G_k(t) \diff t=\psi_k(p_k)$. 
By the defining properties of~$F_k$ and~$G_k$, it is now clear that $F_k(s) = \min\{1, \max\{0, G_k^{-1}(s)\}\}$, which implies that $F_k^{-1}(s)=G_k(s)$ for all $s\in(0,1)$. Integrating both sides of this identity then yields $-\int_{1-p_k}^1 F_k^{-1}(t) \diff t = -\int_{1-p_k}^1 G_k(t) \diff t=\psi_k(p_k)$. By Lemma~\ref{lemma:distributional-regularization}, we may thus conclude that 
\begin{align*}
    \Phi(\bs u;\mac B) &=\max_{\bs p\in\Delta^K} \sum_{k=1}^K p_k u_k + \sum_{k=1}^K \int_{1-p_k}^1 F_k^{-1}(t) \diff t
    =\max_{\bs p \in\Delta^K} \sum_{k=1}^K p_k u_k - \sum_{k=1}^K \psi_k(p_k) = \Phi^R(\bs u;\psi). 
\end{align*}
The claim then follows by taking gradient with respect to~$\bs u$ on both sides.
\end{proof}

Next, we show that there is also a close connection between FTPL and DOPA.

\begin{proposition}[FTPL vs.\ DOPA] \label{prop:dopa-ftpl}
Suppose that $\mathcal{B}$ is a marginal ambiguity set of the form~\eqref{eq:marginal:set} and that the underlying cumulative distribution functions $F_k, k \in[K]$, are continuous and strictly increasing in~$s$ whenever $F_k(s)\in(0,1)$. Then, for every fixed~$\bs u\in\R^K$ there exists $\QQ \in\mathcal P(\R^K)$ that satisfies $\nabla_{\bs{u}} \Phi(\bs{u}; \mac B)=\nabla_{\bs{u}} \Phi^P(\bs{u}; \QQ)$.
\end{proposition}

\begin{proof}
Throughout the proof we use $\bs p=\bs p(\bs u)=\nabla_{\bs u} \Phi(\bs u;\mac B)$ as shorthand for the unique solution of the convex program~\eqref{eq:frechet-reg-max} at the fixed reward estimate~$\bs u$. In addition, we use $\QQ_k\in\mathcal P(\R)$ to denote the unique probability distribution satisfying $\QQ_k(z_k\le s)=F_k(s)$ for all $s\in\R$, and we define 
\begin{equation*}
    \QQ =\sum_{k=1}^K p_k \cdot\left(\otimes_{\ell=1}^{k-1} \QQ_\ell^{-}\right) \otimes \QQ_k^{+} \otimes\left(\otimes_{\ell=k+1}^K \QQ_\ell^{-}\right),
\end{equation*}
where the truncated distributions $\QQ_k^+,\QQ_k^-\in\mathcal P(\R)$ are defined through
\begin{equation*}
    \QQ_k^{+}(z_k\in B)=\QQ_k \left(z_k\in B \left| \,z_k>F_k^{-1}(1-p_k) \right.\right) \quad \text {and} \quad \QQ_k^{-}(z_k\in B)=\QQ_k\left(z_k\in B \left| \, z_k \leq F_k^{-1}(1-p_k) \right.\right)
\end{equation*}
for all Borel sets $B\subseteq\R$, respectively. From the proof of~\citep[Theorem~1]{natarajan2009persistency} we know that~$\QQ$ solves the optimization problem in~\eqref{eq:discrete-best-case}; see also \citep[Proposition~3.6] {ref:taskesen2023semi} for an alternative proof using our notation. The optimality conditions of problem~\eqref{eq:frechet-reg-max} further imply that
\begin{equation}\label{eq:stationary}
    u_k+F_k^{-1}(1-p_k)=u_\ell+F_\ell^{-1}(1-p_\ell) \quad \forall k, \ell \in[K] .
\end{equation}
Next, fix an arbitrary $k\in[K]$,
and note that for every fixed $z_k>F_k^{-1}(1-p_k)$ and $\ell \neq k$ we have
\begin{equation}\label{eq:QQ-}
    \QQ_\ell^{-}\left(z_\ell<z_k+u_k-u_\ell\right) \geq \QQ_\ell^{-}\left(z_\ell \leq F_k^{-1}(1-p_k)+u_k-u_\ell\right)=\QQ_\ell^{-}\left(z_\ell \leq F_\ell^{-1}(1-p_\ell)\right)=1,
\end{equation}
where the first equality follows from~\eqref{eq:stationary}, and the second equality holds because $\QQ_\ell^{-}$ is supported on the interval $(-\infty, F_\ell^{-1}(1-p_\ell)]$. Similarly, for any fixed $z_k \leq F_k^{-1}(1-p_k)$ and $\ell \neq k$ we have
\begin{equation}\label{eq:QQ+}
    \QQ_\ell^{+}(z_\ell<z_k+u_k-u_\ell) \leq \QQ_\ell^{+}\left(z_\ell<F_k^{-1}(1-p_k)+u_k-u_\ell\right)=\QQ_\ell^{+}\left(z_\ell<F_\ell^{-1}(1-p_\ell)\right)=0
\end{equation}
where the first equality follows from~\eqref{eq:stationary}, and the second equality holds because $\QQ_\ell^{+}$ is supported on the interval $(F_\ell^{-1}(1-p_\ell), \infty)$. For any fixed $k\in[K]$, we may thus conclude that
\begin{equation*}
\begin{aligned}
    \QQ \bigg( k \in &\argmax_{\ell\in [K]} u_\ell + z_\ell\bigg) = \QQ (z_\ell<u_k+z_k-u_\ell\ \forall \ell \neq k) \\&=
    p_k \mathbb{E}_{z_k \sim \QQ_k^{+}}\Bigg[\prod_{\ell \neq k} \QQ_\ell^{-}\left(z_\ell<u_k+z_k-u_\ell\right)\Bigg] \\
    &\quad +\sum_{\ell \neq k} p_{\ell} \mathbb{E}_{z_k \sim \QQ_k^{-}}\Bigg[\QQ_{\ell}^{+}(z_\ell<u_k+z_k-u_{\ell}) \prod_{j \neq k, \ell} \QQ_j^{-}(z_j <u_k+z_k-u_j)\Bigg]=p_k.
\end{aligned}
\end{equation*}
Here, the first equality follows from the assumption that the marginal distribution functions~$F_k$, $k \in[K]$, are continuous. This implies that $\QQ^+_\ell$ and~$\QQ^-_\ell$, $\ell\in[K]$, are absolutely continuous to the Lebesgue measure on~$\R$, which in turn implies that $\QQ$ is absolutely continuous with respect to the Lebesgue measure on~$\R^K$. Hence, the event $z_\ell=u_k+z_k-u_\ell$ has zero probability under~$\QQ$. The second equality exploits the construction of~$\QQ$, and the third equality follows from~\eqref{eq:QQ-} and~\eqref{eq:QQ+}.

Finally, the definition of FTPL potential function $\Phi^P(\bs u;\QQ)$ implies that
\begin{equation*}
\begin{aligned}
    \frac{\partial}{\partial{u_k}} \Phi^P(\bs u;\QQ)
    &=\frac{\partial}{\partial{u_k}}\EE_{\bs z \sim \QQ} \bigg[\max_{\ell \in [K]} (u_\ell+z_\ell)\bigg] =\EE_{\bs z \sim \QQ} \bigg[\frac{\partial}{\partial{u_k}}\max_{\ell \in [K]} (u_\ell+z_\ell)\bigg] \\
    &= \EE_{\bs z \sim \QQ} \left[\mathds{1}_{\{k\in \argmax_{\ell\in [K]}(u_\ell + z_\ell)\}} \right]
    =\QQ \left( k \in \argmax_{\ell\in [K]} (u_\ell + z_\ell) \right)=p_k \quad \forall k\in[K],
\end{aligned}
\end{equation*}
where the first equality holds because $\max_{\bs p\in\Delta^K} \bs p^\top (\bs u +\bs z)=\max_{\ell\in[K]} (u_\ell+z_\ell)$, the second equality follows from the dominated convergence theorem, which applies because $\max_{\ell\in[K]} (u_\ell+z_\ell)$ is Lipschitz continuous in~$\bs u$, and the third equality exploits Danskin's theorem~\cite[Proposition~B.25]{bertsekas2016nonlinear} together with the observation that the optimal solution of $\max_{\ell\in [K]}(u_\ell + z_\ell)$ is $\QQ$-almost surely unique. In summary, we have thus shown that $\nabla_{\bs u} \Phi(\bs u;\mac B)=\bs p = \nabla_{\bs u}\Phi^P(\bs u;\QQ)$.
\end{proof}

We are now ready to address the open problem posed in Section~\ref{section:MAB}. The following main theorem bridges the gap between regularization-based and perturbation-based algorithms for MAB problems. It shows that any FTRL algorithm with a convex, smooth and additively separable regularization function~$\psi$ is equivalent to an FTPL algorithm with some noise-sampling distribution~$\QQ$. This insight is consistent with the impossibility result by \cite{ref:kim2019optimality} because the noise terms corresponding to different arms may be correlated under~$\QQ$. The conditions on $\psi$ (especially the additively separability) are restrictive. However, to our best knowledge, the regularization functions of all commonly used FTRL algorithms (such as Tsallis-INF \citep{zimmert2021tsallis}, Exp3 \citep{auer1995gambling} or FTRL with hybrid regularizers \citep{zimmert2019beating}) satisfy these properties. 

% We emphasize that Theorem~\ref{thm:FTRL_vs_FTPL} not only provides an existential result, but in its proof, we explicitly construct such a noise-sampling distribution~$\QQ$.

\begin{theorem}[FTRL vs.\ FTPL]
\label{thm:FTRL_vs_FTPL}
Consider a regularization function $\psi:[0,1]^K\to\R$ defined via $\psi(\bs p)=\sum_{k=1}^K \psi_k(p_k)$, where~$\psi_k:[0,1]\to\R$ is strictly convex and differentiable for every $k\in[K]$. Then, for every $\bs{u}\in\mathbb{R}^K$ there exists a distribution $\QQ \in\mathcal P(\R^K)$ with $\nabla_{\bs{u}} \Phi^R(\bs{u}; \psi)=\nabla_{\bs{u}} \Phi^P(\bs{u}; \QQ)$.
\end{theorem}

\begin{comment}
    
\end{comment}
\begin{proof}
Fix an arbitrary $\bs u\in\R^K$. By Proposition~\ref{prop:dopa-ftrl}, there exists a marginal ambiguity set~$\mathcal B$ of the form~\eqref{eq:marginal:set} such that $\nabla_{\bs u} \Phi^R(\bs u;\psi) = \nabla_{\bs u}\Phi(\bs u;\mac B)$. Proposition~\ref{prop:dopa-ftpl} further implies that there exists a noise distribution $\QQ\in\mac P(\R^K)$ with $\nabla_{\bs u}\Phi(\bs u;\mac B)=\nabla_{\bs{u}} \Phi^P(\bs{u}; \QQ)$. Thus, the claim follows.
\end{proof}

We emphasize that the distribution~$\QQ$ corresponding to a given regularization function~$\psi$ generically depends on the current reward estimate~$\bs u$. In contrast, classical FTPL algorithms use a single noise distribution independent of~$\bs u$. We also emphasize that the proofs of Propositions~\ref{prop:dopa-ftrl} and~\ref{prop:dopa-ftpl} and Theorem~\ref{thm:FTRL_vs_FTPL} are constructive and not merely existential, that is, we provide explicit formulas for the ambiguity set~$\mathcal B$ as well as the noise distribution~$\QQ$ corresponding to~$\psi$.

\section{Regret Analysis of DOPA}
\label{sec:regret-dopa}

A fundamental desideratum in algorithm design is stability. That is, small perturbations in the input of an algorithm should not dramatically alter its output. For example, GBPA with a convex differentiable potential function $\phi$ is stable if the arm-sampling distribution~$\bs p(\bs u)=\nabla_{\bs u} \phi(\bs u)$ is  Lipschitz-continuous in the cumulative reward estimate~$\bs u$. It is well known that adding a convex regularizer to the objective function of a parametric minimization problem improves the stability of its optimal solution \citep{bousquet2002stability}. Improving the stability of FTRL, for instance, is tantamount to reducing the Lipschitz modulus of~$\nabla_{\bs u}\bs{p}(\bs u)$, which can be achieved by increasing the strong convexity constant of the underlying regularization function~$\psi$. This is a direct consequence of the relation $\nabla_{\bs u} \bs p(\bs u)=\rm{diag}(\nabla_{\bs u}^2\phi(\bs u)) \leq \rm{diag}((\nabla_{\bs p}^2 \psi(\bs p))^{-1})$ \citep{penot1994sub,ref:abernethy2015fighting}. 

Stability is a prerequisite for establishing sublinear regret bounds for FTRL algorithms in the adversarial regime~\citep[Section~3.1]{abernethy2014online}. We will now leverage these results to identify conditions under which DOPA enjoys sublinear regret.
%This observation is also shown in the regret decomposition theorem below, which is based upon general theory for FTRL algorithms.
Our analysis will reveal that the regret of DOPA critically depends on the choice of the marginal distribution functions~$F_k$, $k\in[K]$.

\begin{theorem}[Regret analysis of DOPA]
\label{theorem:marginal-zero-mean-Fk}
Suppose that~$\mathcal{B}$ is a marginal ambiguity set of the form~\eqref{eq:marginal:set}. Assume also that the $k$-th marginal distribution function~$F_k$ is differentiable and strictly increasing whenever $F_k(s)\in(0,1)$ and that $\int_\R s F'_k(s)\diff s=0$ for all $k\in[K]$. If $\bs{p}(\bs 0)=\nabla_{\bs u }\Phi(\bs 0; \mathcal{B})$ is the initial arm-sampling distribution and if there exist constants $\gamma\in(1,2)$ and $B>0$ with
\begin{equation}
    F_k' \left(F_k^{-1}(1-p ) \right) \leq B p^\gamma\quad \forall p \in (0, 1),~\forall k \in [K],
\label{eq:post-diff-consistent-cond}
\end{equation}
then the regret of DOPA satisfies
\begin{equation*}
   \mathcal{R}(T) \leq \sum_{k=1}^K \int_{1- p_k(\bs 0)}^1 F_k^{-1}(t) \mathrm{d} t +  \frac{1}{2} B T K^{2-\gamma}
\end{equation*} 
under every possible reward distribution of a non-oblivious adversarial environment.
\end{theorem}

\begin{proof} 
Let $k\opt\in\argmax_{k \in [K]} \mathbb{E}[\sum_{t=1}^T r_{t,k}]$ be the index of an arm with zero regret. Note that~$\hat{r}_{t,k}$ as defined in Algorithm~\ref{alg:mab-gbpa} is an unbiased estimator for~$\mathbb E[r_{t,k}]$ for every arm $k\in[K]$ that has a positive probability $p_k(\hat {\bs u}_{t-1})$ of being selected. Also, recall from Lemma~\ref{lemma:distributional-regularization} that DOPA can be viewed as an FTRL algorithm with a convex regularization function $\psi (\bs p)=-\sum_{k=1}^K \int_{1-p_k}^1 F_k^{-1}(t) \diff t $. Thus, the FTRL regret decomposition in \citep[Theorem~28.10]{lattimore2020bandit} implies~that
\begin{equation*}
\begin{aligned}
   \mathcal{R}(T)
   &\leq \psi(\bs e_{k\opt})-\psi(\bs p(\bs 0)) + \EE \left[\sum_{t=1}^T \left((\bs p_{t+1}- \bs p_{t})^\top \hat{\bs r}_{t} - \mathds D_{\psi}(\bs p_{t+1},\bs p_t)\right) \right],
\end{aligned}
\end{equation*}
where 
%$\psi (\bs p)=-\sum_{k=1}^K \int_{1-p_k}^1 F_k^{-1}(t) \diff t $, while~
$\hat{\bs r}_t$ and $\bs p_t$ are defined as in Algorithm~\ref{alg:mab-gbpa}. From the discussion after Lemma~\ref{lemma:distributional-regularization} we know that~$\psi$ is strictly convex. We may thus use~\citep[Theorems~26.12 \&~26.13]{lattimore2020bandit} to bound the round-$t$ term in the above sum by
\begin{equation*}
    (\bs p_{t+1}- \bs p_{t})^\top \hat{\bs r}_{t} - \mathds D_{\psi}(\bs p_{t+1},\bs p_t) \le \sup_{\lambda\in[0,1]}\frac{1}{2} \hat{\bs r}_{t}^\top \left(\nabla_{\bs p}^2 \psi(\lambda\bs p_{t+1}+(1-\lambda) \bs p_{t}) \right)^{-1} \hat{\bs r}_{t}.
\end{equation*}
The definition of~$\hat{\bs r}_t$ further implies that
\begin{equation*}
\begin{aligned}
    \hat{\bs r}_{t}^\top \left(\nabla_{\bs p}^2 \psi(\lambda\bs p_{t+1}+(1-\lambda) \bs p_{t}) \right)^{-1} \hat{\bs r}_{t} = \frac{r_{t,a_t}^2}{p_{t,a_t}^2} \left( \left(\nabla_{\bs p}^2 \psi(\lambda\bs p_{t+1}+(1-\lambda) \bs p_{t}) \right)^{-1} \right)_{a_t a_t}
\end{aligned}
\end{equation*}
for all $\lambda\in[0,1].$
Finally, observe that $\psi(\bs e_{k\opt})=-\int_{0}^1 F_{k\opt}^{-1}(t)\ \diff t = 0  =-\int_\R sF'_{k\opt}(s)\diff s$ by assumption. Taken together, all of these insights allow us to conclude that
\begin{equation}
\label{eq:regret-bound-general-ftrl}
\begin{aligned}
   \mathcal{R}(T)
   &\leq \sum_{k=1}^K \int_{1-p_k(\bs 0)}^1 F_k^{-1}(t) \diff t  + \frac{1}{2}\EE 
   \left[\sum_{t=1}^T \sup_{\lambda\in[0,1]}\frac{r_{t,a_t}^2}{p_{t,a_t}^2} \left(\nabla_{\bs p}^2 \psi(\lambda\bs p_{t+1}+(1-\lambda) \bs p_{t}) \right)^{-1}_{a_t a_t}\right].
\end{aligned}
\end{equation}
In the remainder of the proof, we will establish an upper bound on the second term in~\eqref{eq:regret-bound-general-ftrl}. Recalling from Algorithm~\ref{alg:mab-gbpa} that $\bs p_t=\bs p(\hat{\bs u}_{t-1})$, $\bs u_t= \bs u_{t-1}+\hat{\bs r}_t$ and $\hat{\bs r}_t= s \bs e_{a_t}$ with $s=-r_{t,a_t}/p_{t,a_t}\geq 0$, we find
\begin{equation}
    \label{eq:p_t-monotonicity}
    p_{t+1,a_t}= p_{a_t}(\hat{\bs u}_{t}) = p_{a_t}(\hat{\bs u}_{t-1}+\hat{\bs r}_t) =p_{a_t}(\hat{\bs u}_{t-1}-s \bs e_{a_t}) \le p_{a_t}(\hat{\bs u}_{t-1}) = p_{t,a_t},
\end{equation}
where the inequality holds because $\Phi(\bs u;\mathcal B)$ is convex in~$\bs u$ such that $p_{a_t}(\bs u) =\partial_{u_{a_t}}\Phi(\bs u;\mathcal B)$ is non-decreasing in~$u_{a_t}$ and because~$s\ge 0$. If $\bs p=\lambda \bs p_{t+1} + (1-\lambda) \bs p_{t}$ for some $\lambda\in[0,1]$, then we have
\begin{equation}\label{expr:psi:diff-consis}
    (\nabla_{\bs p}^2 \psi(\bs p))^{-1}_{a_t a_t} = F_{a_t}'(F_{a_t}^{-1}(1-p_{a_t})) \le B p_{a_t}^\gamma= B (\lambda p_{t+1,a_t} + (1-\lambda) p_{t,a_t})^\gamma \le B p_{t,a_t}^\gamma ,
\end{equation}
where the first equality follows from the definition of $\psi (\bs p)$, which implies via the  inverse function theorem that $\partial^2_{p_k^2}\psi(\bs p) = (F'_k(F^{-1}_k(1-p_k)))^{-1}$. The second equality exploits the definition of~$\bs p$, and the two inequalities follow from~\eqref{eq:post-diff-consistent-cond} and~\eqref{eq:p_t-monotonicity}, respectively. The inequality~\eqref{expr:psi:diff-consis} then implies that
\begin{equation*}
    \begin{aligned}
    & \EE \left[ \sup_{\lambda\in[0,1]}\frac{r_{t,a_t}^2}{p_{t,a_t}^2}\left(\nabla_{\bs p}^2 \psi(\lambda \bs p_{t+1} + (1-\lambda) \bs p_{t})\right)^{-1}_{a_t a_t}\right] 
    \\& \qquad 
    \le \EE \left[ \frac{r_{t,a_t}^2}{p_{t,a_t}^2} B p_{t,a_t}^\gamma \right] 
    = \EE \left[ \EE\left[ \left. \frac{r_{t,a_t}^2}{p_{t,a_t}^2} B p_{t,a_t}^\gamma \right| \hat{\bs u}_{t-1},\bs r_t\right]\right] 
    = \EE \left[B \sum_{k=1}^K r_{t,k}^2 p_{t,k}^{\gamma-1} \right] \le \EE \left[B \sum_{k=1}^K p_{t,k}^{\gamma-1} \right],
    \end{aligned}
\end{equation*}
where the first equality exploits the law of iterated conditional expectations, and the second equality holds because~$a_t=k$ with probability~$p_{t,k}$ conditional on~$\hat{\bs u}_{t-1}$ and~$\bs r_t$. The second inequality holds because~$r_{t, k}^2\in [0,1]$ for all $k\in[K]$. Next, note that the $1/(2-\gamma)$-norm and the $1/(\gamma-1)$-norm are mutually dual for every $\gamma \in (1, 2)$. Hölder's inequality thus implies that
\begin{equation*}
\begin{aligned}
    \sum_{k=1}^K p_{t,k}^{\gamma-1} \leq\left(\sum_{k=1}^K p_{t,k}^{\frac{\gamma-1}{\gamma-1}}\right)^{\gamma-1}\left(\sum_{k=1}^K 1^{\frac{1}{2-\gamma}}\right)^{2-\gamma} =K^{2-\gamma}.
\end{aligned}
\end{equation*}
This observation completes the proof. 
\end{proof}

\section{Optimality of DOPA} \label{sec:optimal}
We now aim to identify marginal ambiguity sets~$\mathcal{B}$ for which DOPA achieves optimal regret guarantees across different regimes. 
%Within these specific settings, the unresolved conjectures that DOPA addresses become particularly relevant concerning the recoverability of FTRL algorithms through the application of FTPL methods. 
We will see that such optimal regret guarantees are available for certain Fr\'echet ambiguity sets~$\mathcal B$ that are defined in terms of a marginal generator~$F$.

\begin{definition}[Marginal generator]\label{def:marginal-generator}
A marginal generator~$F:\R\to\R$ is a strictly increasing differentiable function with $\lim_{s\to-\infty} F(s)\leq 0$, $\lim_{s\to+\infty} F(s)\geq 1$ and~$\int_0^1 F^{-1}(t)\diff t=0$.
\end{definition}

\begin{definition}[Fr\'echet ambiguity set]
    \label{def:frechet-ambiguity-set} A Fr\'echet ambiguity set~$\mathcal{B}$ is a marginal ambiguity set of the form~\eqref{eq:marginal:set}, where the marginal cumulative distribution functions are defined through
    \begin{equation}
        \label{eq:marginal_dists}
        F_k(s) = \min\big\{1, \max\{0, 1 -F(-s / \eta_k)\} \big\}\quad\forall k\in[K]
    \end{equation}
    for some vector $\bs \eta \in\R_{++}^K$ and some marginal generator~$F$.
\end{definition}

Before analyzing the regret of DOPA under Fr\'echet ambiguity sets, we present an auxiliary result that relates any potential function of the form~$\Phi(\bs u ; \mathcal{B})$ induced by some Fr\'echet ambiguity set~$\mathcal B$ to a potential functions of the form $\Phi^R(\bs u; \psi)$ induced by some regularization function $\psi$. We will see that both $\mathcal B$ and $\psi$ are uniquely determined by a vector~$\bs \eta$ and a marginal generator~$F$.

\begin{theorem}[Fr\'echet regularization]
\label{thm:frechet-regularization}
Suppose that $\mathcal{B}$ is a Fr\'echet ambiguity set in the sense of Definition~\ref{def:frechet-ambiguity-set} induced by some $\bs \eta \in\R_{++}^K$ and some marginal generator~$F$. If $f(s) = \int_{0 }^{s} F^{-1}(t) \diff t$ for all $s\in[0,1]$ and $\psi(\bs p) = \sum_{k=1}^K \eta_k f(p_k)$, then we have $\Phi(\bs u ; \mathcal{B})=\Phi^R(\bs u; \psi)$.
\end{theorem}

As the marginal generator~$F$ is strictly increasing and as its range covers the open interval~$(0,1)$, the inverse function~$F^{-1}(t)$ is well-defined for for all $t\in(0,1)$, which in turn implies that $f(s)$ is well-defined for all $s\in[0,1]$. We trivially have~$f(0)=0$. In addition, $f(s)$ is smooth and convex (because~$F^{-1}$ inherits the monotonicity of~$F$), and we have $f(1)=0$ (because $\int_0^1 F^{-1}(t)\diff t =0$).

\begin{proof}[Proof of Theorem~\ref{thm:frechet-regularization}.]
    As the marginal generator $F$ is strictly increasing, the definition of $F_k$ in~\eqref{eq:marginal_dists} implies that $F_k^{-1}(x) = -F^{-1}(1-x)\eta_k$ for all $x \in (0, 1)$. By the definition of~$f$, we thus find
    \begin{equation}
    \label{eq:f-formula}
    \begin{aligned}
        f(s) = \int_{0}^{s} F^{-1}(t) \diff t &= - \int_{1}^{1 - s } F^{-1} \left( {1 - x} \right) \diff x
        = -\frac{1}{ \eta_k} \int_{1 - s}^1 F_k^{-1}(x) \diff x\quad \forall s\in[0,1],
    \end{aligned}
    \end{equation}
    where the second equality follows from the substitution $x\leftarrow 1- t$. By Lemma~\ref{lemma:distributional-regularization}, we thus obtain 
    \[
        \Phi(\bs u; \mathcal{B})= \max\limits_{\bs p\in \Delta^K}  \sum\limits_{k=1}^K u_k p_k - \sum\limits_{k=1}^K \eta_k \,f( {p_k})=\max\limits_{\bs p\in \Delta^K}  \sum\limits_{k=1}^K u_k p_k - \psi(\bs p)=\Phi^R(\bs u,\psi).
    \]
    This observation completes the proof. \end{proof}

We are now ready to analyze the regret of DOPA under Fr\'echet ambiguity sets. 

\begin{theorem}[Regret analysis of DOPA with Fr\'echet ambiguity sets]
\label{thm:frechet:regret}
Suppose that all conditions of Theorem~\ref{thm:frechet-regularization} hold and that $\bs \eta=\eta \bs{1}$ for some~$\eta>0$. If there exist $\gamma\in(1,2)$ and $C>0$ with
\begin{equation}\label{eq:cond:diff-consis:f}
    F'(F^{-1}( p)) \le C p^{\gamma}  \quad\forall p\in(0,1),
\end{equation}
then the regret of DOPA satisfies
\begin{equation*}
   \mathcal{R}(T) \leq -\eta K f(1/K) + \frac{C T K^{2-\gamma}}{2\eta}\quad \forall T\in\mathbb{N}
\end{equation*}
under every possible reward distribution of a non-oblivious adversarial environment.
\end{theorem}

\begin{proof} %[Proof of Corollary~\ref{corollary:frechet:regret}]
We first show that all conditions of Theorem~\ref{theorem:marginal-zero-mean-Fk} are satisfied. To this end, select any arm~$k\in[K]$. Thanks to the assumed properties of~$F$, the distribution function $F_k$ as defined in~\eqref{eq:marginal_dists} is differentiable and strictly increasing whenever $F_k(s)\in(0,1)$. In addition, observe that
\[
    \int_\R s F_k'(s) \diff s = \int_0^1 F_k^{-1}(x) \diff x =- \eta \int_{0}^{1} F^{-1}(1-x) \diff x = - \eta\int_{0}^{1} F^{-1} (t) \diff t =0,
\]
where the first equality follows from the substitution $F_k^{-1}(x)\leftarrow s$, the second equality holds because~\eqref{eq:marginal_dists} implies that $F_k^{-1}(x) = -\eta F^{-1}(1-x)$ for all $x \in (0, 1)$,  the third equality exploits the substitution $t\leftarrow 1- x$, and the last equality holds by assumption. Furthermore, we have 
\begin{equation*}
    \begin{aligned}
    F_k'(F_k^{-1}(1- p_k))
    &=\left(\nabla_{\bs p}^2 \left(-\sum_{k=1}^K \int_{1-p_k}^1 F_k^{-1}(t) \diff t \right)\right)^{-1}_{kk}
    \\&=\left(\nabla_{\bs p}^2 \left(\sum\limits_{k=1}^K \eta \int_{0}^{p_k} F^{-1}(t)\diff t\right)\right)^{-1}_{kk}
% =F_k'(F_k^{-1}(1- p))
 =\frac{1}{\eta}F'(F^{-1}( p))
\leq \frac{C}{\eta} p_k^\gamma \quad\forall p_k\in[0,1],
    \end{aligned}
\end{equation*}
where the first and third equalities follow from the inverse function theorem, and the second equality exploits again the definition of~$F_k$ in~\eqref{eq:marginal_dists} and a simple variable substitution. The inequality, finally, follows from the assumption~\eqref{eq:cond:diff-consis:f}. Thus, $F_k$ satisfies~\eqref{eq:post-diff-consistent-cond} with $B=C/\eta$. As $k\in[K]$ was chosen arbitrarily, we have now verified all conditions of Theorem~\ref{theorem:marginal-zero-mean-Fk}. We may thus conclude that
\begin{equation}
   \mathcal{R}(T) \leq \sum_{k=1}^K \int_{1-p_k(\bs 0)}^1 F_k^{-1}(t) \diff t + \frac{C T K^{2-\gamma}}{2\eta}= -\eta \sum_{k=1}^K f( {p_k(\bs 0)}) + \frac{C T K^{2-\gamma}}{2\eta}.
   \label{eq:p0-max-f}
\end{equation}
It remains to be shown that $\bs{p}(\bs 0)=\nabla_{\bs u }\Phi(\bs 0; \mathcal{B}) = \mathbf{1}/K$ is the unique maximizer of problem~\eqref{eq:frechet-reg-max} at~$\bs u=\bs 0$. By the formula~\eqref{eq:f-formula} for~$f$ in the proof of Theorem~\ref{thm:frechet-regularization}, problem~\eqref{eq:frechet-reg-max} at~$\bs u=\bs 0$ has the same maximizers as $\max_{\bs p \in \Delta^K} H(\bs p)$, where $H(\bs p)=-\sum_{k=1}^K f( p_k)$ is shorthand for the rescaled objective function. This problem is solvable thanks to Weierstrass' maximum theorem, which applies because~$H(\bs p)$ is continuous (in fact smooth) and~$\Delta^K$ is compact. In the following, we use $\Pi_K$ to denote the group of all permutations of~$[K]$. Note that both $H(\bs p)$ as well as the feasible set~$\Delta^K$ are permutation symmetric. Hence, if $\bs p\opt$ is a maximizer, then so is $\bs p\opt_\pi=(p\opt_{\pi(1)},\ldots, p\opt_{\pi(K)})$, for any~$\pi\in\Pi_K$. As~$\Delta^K$ is convex and as $|\Pi_K|=K!$, it is clear that the uniform convex combination $\bar{\bs p}\opt =\frac{1}{K!} \sum_{\pi\in\Pi_K}\bs p\opt_\pi$ belongs to the feasible set $\Delta^K$, too. In addition, the objective function value of~$\bar{\bs p}\opt$ satisfies
\[
    H(\bar{\bs p}\opt)\geq \frac{1}{K!} \sum_{\pi\in\Pi_K} H(\bs p\opt_\pi) = \max_{\bs p\in\Delta^K} H(\bs p).
\]
Here, the inequality follows from Jensen's inequality, which applies because~$f$ is convex and~$H$ is concave. This implies that~$\bar{\bs p}\opt$ is also an optimal solution. By construction, $\bar{\bs p}\opt$ is invariant under permutations of its elements, which allows us to conclude that $\bar{\bs p}\opt=\bs 1/K$. We know from Lemma~\ref{lemma:distributional-regularization} that the maximizer of problem~\eqref{eq:frechet-reg-max} is unique. In summary, we have thus verified that $\bs{p}(\bs 0)=\nabla_{\bs u }\Phi(\bs 0; \mathcal{B}) = \mathbf{1}/K$ is indeed the unique maximizer of~\eqref{eq:frechet-reg-max}. The claim then follows from~\eqref{eq:p0-max-f}. 
\end{proof}

It is well known that the optimal regret in the adversarial regime is of the order $\mathcal{O}(\sqrt{KT})$ \citep[Theorem~1]{ref:audibert2009minimax} and is achieved by an FTRL algorithm with a Tsallis entropy regularizer \citep[Theorem~3.1]{ref:abernethy2015fighting}. The next corollary of Theorem~\ref{thm:frechet:regret} identifies a Fr\'echet ambiguity set for which DOPA offers the same optimal regret guarantee.
 
\begin{corollary}[Optimality of DOPA]
\label{cor:gbpa-tsallis}
Suppose that $\mathcal{B}$ is a Fr\'echet ambiguity set and that the marginal generator is a shifted Pareto distribution of the form~$F(s) = (1/ \alpha - s (1-\alpha) / \alpha)^{-\frac{1}{1-\alpha}}$ with $\alpha \in (0,1)$. Then, the regret of DOPA with $\bs \eta=\eta\mathbf{1}$ and $\eta=\sqrt{(T(1-\alpha))/(2 \alpha)} K^{\alpha-\frac{1}{2}}$ satisfies 
\[
    \mathcal{R}(T) \leq \sqrt{KT/(\alpha(1-\alpha))}
\]
under every possible reward distribution of a non-oblivious adversarial environment.
\end{corollary}

Corollary~\ref{cor:gbpa-tsallis} asserts that if $\mathcal{B}$ is a Fr\'echet ambiguity set generated by a shifted Pareto distribution, then DOPA with a learning rate~$\eta$ that is adapted to~$T$ attains the optimal adversarial regret~$\mathcal{O}(\sqrt{KT})$.

%If $\alpha=1/2,$ we have $\mathcal{R}(T) \leq 2\sqrt{2KT}.$

\begin{proof}[Proof of Corollary~\ref{cor:gbpa-tsallis}.]
Observe first that~$F$ is indeed a marginal generator in the sense of Definition~\ref{def:marginal-generator}. From Theorem~\ref{thm:frechet-regularization} we thus know that $\Phi({\bs u}; \mathcal{B}) = \Phi^R(\bs u; \psi)$, where $\psi(\bs p) = \sum_{k=1}^K \eta f(p_k) =\eta \sum_{k=1}^K \int_{0 }^{p_k} F^{-1}(t) \diff t$.
Thanks to our specific choice of~$F$, we have
\[
    \int_{0 }^{p_k} F^{-1}(t) \diff t =  \int_{0 }^{p_k} \frac{1-\alpha t^{\alpha-1}}{1-\alpha}\diff t = \frac{p_k-p_k^{\alpha}}{1-\alpha} \quad \forall k\in[K].
\] 
This implies that $\psi = \eta\psi_\alpha^{\mathds T}$, where $\psi_\alpha^{\mathds T}$ is the Tsallis entropy with parameter~$\alpha$ defined in~\eqref{eq:tsallis-entropy}. In addition, one readily verifies that $F'(F^{-1}(p))=p^{2-\alpha}/\alpha$. Thus, all conditions of Theorem~\ref{thm:frechet:regret} are satisfied with $C=1/\alpha$ and $\gamma=2-\alpha$, and we may conclude that 
\begin{equation*}
   \mathcal{R}(T) \leq -\eta K f(1/K) + \frac{C T K^{2-\gamma}}{2\eta}
   = \eta \frac{K^{1-\alpha}-1}{1-\alpha} + \frac{T K^{\alpha}}{2\alpha\eta}=\sqrt{\frac{KT}{\alpha(1-\alpha)}}
   \quad \forall T\in\mathbb{N},
\end{equation*} 
where the second equality exploits our specific choice of~$\eta.$ Hence, the claim follows.
\end{proof}

%, resolving the open question discussed on page 7~\citep{ ref:honda23follow, ref:kim2019optimality}.

The regret analysis of DOPA developed in Corollary~\ref{cor:gbpa-tsallis} is arguably simpler than that of optimal FTPL methods, which require subtle probabilistic arguments \citep{ref:honda23follow}. As it only uses basic tools from convex analysis instead of non-standard concepts from variational analysis such as sub-Hessians, it is even somewhat simpler than the regret analysis of the optimal FTRL method with Tsallis entropy regularizer in \citep[Theorem~3.1]{ref:abernethy2015fighting}---despite many similarities.

%enjoys a streamlined FTRL-type regret analysis, as opposed to more involved analysis of Note that the regret of FTRL with Tsallis entropy regularizer is already known, and Corollary~\ref{cor:gbpa-tsallis} thus recovers~\citep[Theorem~3.1]{ref:abernethy2015fighting}. Our proof is arguably simpler than~\citep[Theorem~3.1]{ref:abernethy2015fighting} because we avoid the discussion of sub-Hessians.

Recall from Remark~\ref{remark:exp3:iid} that DOPA reduces to the Exp3 algorithm if~$\mathcal{B}$ is a singleton containing a product Gumbel distribution. The following remark highlights that the Exp3 algorithm is also recovered from DOPA if~$\mathcal{B}$ is a Fr\'echet ambiguity set with an exponential marginal generator.

\begin{remark}[Exp3 algorithm revisited]
Suppose that~$\mathcal{B}$ is a Fr\'echet ambiguity set with $\bs \eta =\eta\bs 1 $ for some~$\eta>0$ and with marginal generator $F(s)=\exp(s-1)$. In this case, Theorem~\ref{thm:frechet-regularization} implies that DOPA is equivalent to FTRL with regularization function 
\[
    \psi(\bs p)=\eta \sum_{k=1}^K \int_{0 }^{p_k} F^{-1}(t) \diff t=\eta \sum_{k=1}^K \int_{0 }^{p_k} (\log(t)+1) \diff t=\eta \sum_{k=1}^K p_k\log(p_k),
\]
which is in turn known to be equivalent to the Exp3 Algorithm; see \citep[Section~3]{ref:abernethy2015fighting}. This can also be checked directly. Indeed, by inspecting the optimality conditions of the convex program $\max_{\bs p \in \Delta^K} \bs p^\top \bs u - \psi(\bs p)$, one readily
verifies that the corresponding arm-sampling probabilities are given by $p_k(\bs u)=\exp(u_k/\eta )/(\sum_{j=1}^K \exp( u_j/\eta ))$ for all $k\in[K]$. However, these are precisely the arm-sampling probabilities of the Exp3 algorithm by~\citet{auer1995gambling}.
\end{remark}

Corollary~\ref{cor:gbpa-tsallis} relies on the implicit assumption that the learner knows the duration~$T$ of the game {\em ex ante} and is thus able to choose a learning rate~$\eta$ that adapts to~$T$. In the remainder of this section we will show that DOPA can offer optimal regret guarantees even if~$T$ is unknown and even if there is uncertainty about the adversarial power of the environment. To this end, we study a generalized {\em anytime GBPA} algorithm that runs over an indefinite number of rounds; see Algorithm~\ref{alg:mab-gbpa-anytime}.

\begin{algorithm}[H]
\caption{Anytime GBPA}
\begin{algorithmic}
\REQUIRE Differentiable convex functions $(\phi_t)_{t\in\mathbb{N}}$ with $\nabla_{\bs u} \phi_t(\bs u) \in \Delta^K$ 
% and $\nabla_{\bs u} \Phi(\bs u)_k>0$ for all $k \in [K]$ \\
\STATE Initialize $\hat {\bs u}_{0} = \bs 0$
\FOR{round $t \in\mathbb{N}$}
    \STATE Environment chooses a reward vector~$\bs r_t \in [-1,0]^K$
    \STATE Learner chooses $a_t \sim \bs  p_t = \nabla_{\bs u} \phi_t(\bs u)|_{\bs u = \hat{\bs u}_{t-1}}$ and receives reward~$r_{t,a_t}$
    % \STATE Learner receives $r_{t, a_t}$
    \STATE Learner estimates single-round reward vector~$\hat{\bs r}_t = (r_{t, a_t} / p_{t, a_t}) \bs e_{a_t}$
    \STATE $\hat{\bs u}_t \gets \hat{\bs u}_{t-1} + \hat{\bs r}_t$
    \ENDFOR
\end{algorithmic}
\label{alg:mab-gbpa-anytime}
\end{algorithm}

Algorithm~\ref{alg:mab-gbpa-anytime} extends Algorithm~\ref{alg:mab-gbpa} in that it runs forever and allows the potential function~$\phi_t$ to change with~$t$. Below, we will thus study a variant of DOPA with a time-dependent ambiguity set~$\mathcal B_t$. 

The optimal regret guarantees of any bandit algorithm depend on the adversarial power of the environment. Hence, an algorithm that attains the optimal regret in the non-oblivious adversarial regime is not guaranteed to attain the optimal regret in the stochastic regime, say. In order to present versions of DOPA that are simultaneously optimal across different learning regimes, we henceforth describe the adversarial power of the environment in a unifying manner via a self-bounding constraint \citep{zimmert2021tsallis}. Formally, we thus assume that for any reward distribution available to the environment there exist $\bs{\Delta} \in [0,1]^K$ and $C \geq 0$ such that the inequality
\begin{equation}\label{eq:self-bounding}
    \mathcal{R}(T) \geq \sum_{t=1}^T \sum_{k=1}^K \Delta_k \PP(a_t=k) - C 
\end{equation}
holds for all planning horizons $T\in\mathbb N$ and for all admissible arm-sampling distributions.
%Note that the above condition must be satisfied as soon as the algorithm terminates, but is not required to hold while the algorithm is running. 
For example, in the stochastic bandit setting, the reward vectors~$\bs r_{t}$ are drawn independently from some fixed distribution on~$[K]$. In this case, the regret can be written as 
\begin{equation*}
    \begin{aligned}
        \mathcal{R}(T) = \sum_{t=1}^T \!\sum_{k=1}^K \left(\max_{\ell \in [K]} \mathbb{E}[r_{t,\ell}] - \mathbb{E}[r_{t,a_t}| a_t=k]\right) \PP(a_t=k)
    \end{aligned}
\end{equation*}
and thus satisfies~\eqref{eq:self-bounding} with \( \Delta_ k  = \max_{\ell \in [K]} \mathbb{E}[r_{t,\ell}] - \mathbb{E}[r_{t,k}] \) for all~$k \in [K]$ and \( C = 0 \).
Similarly, one can show that \eqref{eq:self-bounding} holds in the stochastically constrained adversarial~\citep{ref:wei2018more} and the adversarially corrupted stochastic~\citep{ref:lykouris2018stochastic} learning regimes. Exploiting the equivalence of FTRL with Tsallis entropy regularization and DOPA with shifted Pareto marginals, we can now show that the anytime version of DOPA inherits the BOBW capability of FTRL. 

\begin{theorem}[BOBW capability of DOPA]
    \label{theorem:marginal-anytime}
    Suppose that $\mathcal{B}_t$ is a time-dependent Fr\'echet ambiguity set in the sense of Definition~\ref{def:frechet-ambiguity-set} with $\bs \eta=\eta_t \bs 1$, $\eta_t=2\sqrt{t}$ and marginal generator $F(s)=(2-s)^{-2}$ for all $t\in\mb N$. Then, the regret of DOPA satisfies $\mathcal{R}(T) \leq 4\sqrt{KT}+1$ for all $T\in\mb N$ and under all reward distributions of a non-oblivious adversarial environment. 
    % In addition, if the environment obeys~\eqref{eq:self-bounding}, then the regret of DOPA satisfies 
    % \begin{equation}
    %     \label{eq:DOPA-log-regret}
    %     \mathcal{R}(T) \le \mathcal{O}\left(\sum_{k\in[K]: \Delta_k>0} \log(T)/\Delta_k \right) \quad\forall T\in\mathbb N.
    % \end{equation}
    In addition, the regret of DOPA satisfies 
    \begin{equation}
        \label{eq:DOPA-log-regret}
        \mathcal{R}(T) \le \mathcal{O}\left(\sum_{k\in[K]: \Delta_k>0} \log(T)/\Delta_k \right) \quad\forall T\in\mathbb N
    \end{equation}
    under all reward distributions of an environment constrained by~\eqref{eq:self-bounding}.
\end{theorem}

Note that the marginal generator $F(s)=(2-s)^{-2}$ coincides with the shifted Pareto distribution corresponding to $\alpha=1/2$ from Corollary~\ref{cor:gbpa-tsallis}. Theorem~\ref{theorem:marginal-anytime} implies that DOPA achieves the optimal $\mathcal{O}(\log T)$ regret in the stochastic regime and the optimal $\mathcal{O}(\sqrt{KT})$ regret in the adversarial regime. Thanks to the time-dependent learning rate $\eta_t=2\sqrt{t}$, DOPA also displays the anytime property, that is, it attains optimal regret bounds for every time horizon~$T$ without requiring knowledge of~$T$.

\begin{proof} Theorem~\ref{thm:frechet-regularization} implies that $\Phi( \bs u; \mathcal{B}_t) = \Phi^R(\bs u
; \psi_t)$, where $\psi_t(\bs p) = \eta_t\sum_{k=1}^K \int_{0 }^{p_k} F^{-1}(t) \diff t$. The proof of Corollary~\ref{cor:gbpa-tsallis} further implies that $\psi_t = \eta_t\psi_{1/2}^{\mathds T}$, where $\psi_{1/2}^{\mathds T}$ stands for the Tsallis entropy 
with parameter~$1/2$. By \cite[Theorem~1]{zimmert2021tsallis}, which applies to Tsallis-regularized FTPL algorithms with adaptive learning rate $\eta_t=2\sqrt{t}$, we may then conclude that $\mathcal{R}(T) \leq 4\sqrt{KT}+1$ for every~$T\in\mb N$. If the adversary selects rewards that satisfy~\eqref{eq:self-bounding} and if $\Delta_k>0$ for some $k\in[K]$, then \cite[Theorem~2]{ito2021parameter} further ensures 
that~\eqref{eq:DOPA-log-regret} holds. Hence, the claim follows.
\end{proof}

The following corollary shows that DOPA can even recover FTRL schemes with hybrid regularizers. This is achieved by studying harmonic averages of two different marginal generators.

% {\color{blue}}
%In addition to recovering FTRL using a regularization function derived from a single marginal generator function, DOPA also effectively recovers hybrid regularizers. 

\begin{corollary}[Hybrid Fr\'echet regularizers]
    \label{corollary:hybrid-gamma-eta-k}
    Select any weights $\gamma_1,\gamma_2>0$ and any marginal generators~$G_1$ and~$G_2$. Suppose that~$\mathcal{B}$ is a Fr\'echet ambiguity set with~$\bs \eta =\mathbf{1}$ and with marginal generator $F(s)= (\gamma_1 G_1^{-1}+\gamma_2 G_2^{-1})^{-1}(s)$. If $g_1(s) =  \int_{0 }^{s} G_1^{-1}(t) \diff t$ and $g_2(s)=  \int_{0 }^{s} G_2^{-1}(t) \diff t$ for all $s\in\R$, then $\Phi(\bs u ; \mathcal{B})=\Phi^R(\bs u; \psi)$, where $\psi(\bs p) = \sum_{k=1}^K (\gamma_1 g_1(p_k) + \gamma_2 g_2(p_k))$.
\end{corollary}

Corollary~\ref{corollary:hybrid-gamma-eta-k} shows that any FTRL method with a hybrid regularizer representable as a sum of two convex functions can be interpreted as an instance of DOPA with a Fr\'echet ambiguity set induced by a harmonic average of two marginal generators. In conjunction with Theorem~\ref{thm:FTRL_vs_FTPL}, this result implies that we can systematically construct FTPL algorithms that are equivalent to FTRL algorithms with hybrid regularizers, some of which are known to display attractive BOBW capabilities. Corollary~\ref{corollary:hybrid-gamma-eta-k} thus addresses an open problem posed by \citet{ref:honda23follow}, who state that {\em ``it would be a very challenging task to realize the effect of hybrid regularization by FTPL.''} 

% {\color{blue}FTRL methods with hybrid regularizers constitute the only class of FTRL methods that achieves BOBW for combinatorial semi-bandits, where the learner is allowed to select multiple arms~\citep{zimmert2019beating}.}

\begin{proof}[Proof of Corollary~\ref{corollary:hybrid-gamma-eta-k}.]
By Lemma~\ref{lemma:distributional-regularization}, we have
\begin{equation*}\begin{aligned}
    \Phi(\bs u ; \mathcal{B})= \max\limits_{ \bs p\in \Delta^K} \sum\limits_{k=1}^K u_k p_k + \sum_{k=1}^K \displaystyle\int_{1-p_k}^1 F_k^{-1}(t)\diff t.
\end{aligned}\end{equation*}
The definitions of~$g_1$ and $g_2$ further imply that
\begin{equation*}
\begin{aligned}
    \gamma_1 g_1(s) + \gamma_2 g_2(s) &= \gamma_1 \int_{0 }^{s}  G_1^{-1}(t) \,\diff t +  \gamma_2 \int_{0 }^{s}  G_2^{-1}(t)\, \diff t \\
    &= \int_{1 - s }^{1} (\gamma_1 G_1^{-1}+\gamma_2 G_2^{-1})(1-x)\,\diff x \\
    &= \int_{1 - s }^{1}F^{-1}(1-x)\,\diff x = -\int_{1 - s}^1 F_k^{-1}(x) \,\diff x \quad \forall k\in[K],
\end{aligned}
\end{equation*}
where the second and the third equalities follow from the variable substitution $x\leftarrow 1- t$ and the definition of~$F$, respectively. The last equality exploits the definition of~$F_k$ in~\eqref{eq:marginal_dists} %which implies that $F_k^{-1}(x) = -F^{-1}(1-x)\eta_k$ for all $x \in (0, 1)$
and the assumption that $\bs \eta=\bs 1$. Combining the above derivations and recalling the definition of~$\psi$ then yields
\[
    \Phi(\bs u ; \mathcal{B})= \max_{\bs p\in \Delta^K}  \sum_{k=1}^K u_k p_k - \sum_{k=1}^K (\gamma_1 g_1(p_k) + \gamma_2 g_2(p_k)) =\Phi^R(\bs u; \psi).
\]
Hence, the claim follows.
\end{proof}

All optimal FTRL methods with hybrid regularizers studied to date assume that the regularizer consists of a sum of merely {\em two} elementary convex functions \citep{jin2024improved,zimmert2019beating}. More versatile FTRL methods can be obtained by generalizing Corollary~\ref{corollary:hybrid-gamma-eta-k} in the obvious way, that is, by setting the regularization function~$\psi$ to the sum of the integrals of~$N>2$ inverse marginal generators $G_1^{-1},\ldots, G_N^{-1}$. The resulting FTRL method can then be interpreted as a version of DOPA whose Fr\'echet ambiguity set is generated by the harmonic mean of $G_1,\ldots, G_N$.

%This construction could potentially be useful when trying to find FTPL methods that corresponds to even more general class of FTRL methods.

To close this section, we use Corollary~\ref{corollary:hybrid-gamma-eta-k} to show that DOPA with a Fr\'echet ambiguity set generated by two marginal generators can achieve theoretically optimal BOBW guarantees. 

\begin{corollary}[Adaptive hybrid Fr\'echet regularizers]
\label{corollary:hybrid:tsallis-shannon}
    Suppose that $\mathcal{B}_t$ is a time-dependent Fr\'echet ambiguity set in the sense of Definition~\ref{def:frechet-ambiguity-set} with~$\bs \eta =\mathbf{1}$ and marginal generator $F(s)=(\gamma_{t} G_1^{-1}+\gamma_{t} G_2^{-1})^{-1}(s)$, where $\gamma_{t}=\sqrt{t}$, $G_1(s) = 1-\exp(-(s+1))$ and $G_2(s)=(-2s)^{-2}$ for all $t\in\mathbb N$. Then, the regret  of DOPA satisfies $\mathcal{R}(T) \le \mathcal{O}(\sqrt{KT})$ 
    for all $T\in\mb N$ and under all reward distributions of a non-oblivious adversarial environment. 
    In addition, the regret of DOPA then also satisfies
    \[ \mathcal{R}(T) \le \mathcal{O}\left(\sum_{k\in[K]: \Delta_k>0}\log T/\Delta_k \right)+\mathcal{O}\left(\sum_{k\in[K]: \Delta_k>0}(\log K)^2/\Delta_k \right) \quad\forall T\in\mathbb N
    \]
    under all reward distributions of an environment constrained by~\eqref{eq:self-bounding} with $|\{k\in[K]:\Delta_k=0\}|=1$.
\end{corollary}

\begin{proof} 
Corollary~\ref{corollary:hybrid-gamma-eta-k} readily implies that $\Phi(\bs u ; \mathcal{B}_t)=\Phi^R(\bs u; \psi_t)$ for every $t\in[T]$, where $\psi_t(\bs p) = \sum_{k=1}^K (\gamma_{t} g_1(p_k) + \gamma_{t} g_2(p_k))$, $g_1(s) =  \int_{0 }^{s} G_1^{-1}(t)\, \diff t $ and $g_2(s)=  \int_{0 }^{s} G_2^{-1}(t)\, \diff t$. In addition, one readily verifies that $G_1^{-1}(t)=-1-\log(1-t)$ and $G_2^{-1}(t) = -(2\sqrt{t})^{-1}$, which implies that $\gamma_{t} g_1(s) + \gamma_{t} g_2(s)=-\gamma_{t} (\sqrt{s}+(s-1)\log (1-s))$. Hence, the version of DOPA at hand is equivalent to an FTRL method with regularizer $\psi_t(\bs p)=\sum_{k=1}^K -\sqrt{t} (\sqrt{p_k}+(p_k-1)\log (1-p_k))$. The claim then follows from general results on FTRL algorithms with hybrid regualrizers \cite[Theorem~3]{zimmert2019beating}.
\end{proof}

The results of this section imply via Lemma~\ref{lemma:distributional-regularization} and Proposition~\ref{prop:dopa-ftpl} that there exist FTPL algorithms that are equivalent to FTRL algorithms with hybrid regularizers (in particular optimal~ones). 

%To the best of our knowledge, it has not been known whether any 

\section{Computational Efficiency of DOPA}
\label{sec:experiments}
FTPL algorithms are popular primarily because of their computational efficiency. Indeed, the arm~$a_t$ to be pulled in round~$t$ is found by sampling~$\bs z\sim\QQ$ and then identifying the largest component of the perturbed reward estimate $\hat{\bs u}_{t-1}+ \bs z$. Recall that the components of~$\bs z$ are usually assumed to be i.i.d.\ under~$\QQ$. Hence, the per-iteration complexity of FTPL is of the order $\mathcal O(K)$. In contrast, FTRL algorithms need to solve a convex optimization problem in each round~$t$, which imposes a significantly higher computational burden. From Theorem~\ref{thm:FTRL_vs_FTPL} we know, however, that every FTRL algorithm induced by an additively separable regularization function~$\psi$ is equivalent to an FTPL algorithm induced by some disturbance distribution~$\QQ$. This connection between FTRL and FTPL is mediated by DOPA. Specifically, the noise distribution~$\QQ$ is a solution of the optimization problem~\eqref{eq:discrete-best-case} and thus changes with the reward estimate $\bs u = \hat{\bs u}_{t-1}$. The FTPL algorithm corresponding to a given FTRL algorithm thus needs to solve an instance of~\eqref{eq:discrete-best-case} in each round~$t$ in order to compute the current noise distribution. Hence, it appears that all computational advantages of FTPL {\em vis-\`a-vis} FTRL are outweighed by the time needed to solve just another optimization problem.

We will now show that this suspicion is unwarranted. Instead of computing~$\QQ$ by solving problem~\eqref{eq:discrete-best-case} at $\bs u=\hat{\bs u}_{t-1}$ and then sampling~$a_t$ from $\bs p=\nabla_{\bs u} \Phi^P(\bs u; \QQ)$ by drawing a sample~$\bs z$ from~$\QQ$, we propose here to compute the arm-sampling distribution $\bs p$ {\em directly}. This can be done highly efficiently by recalling from Proposition~\ref{prop:dopa-ftpl} that $\bs p= \nabla_{\bs u} \Phi(\bs u; \mathcal{B})$ and by leveraging a bisection method inspired by \cite[Algorithm~2]{ref:taskesen2023semi} for computing $\nabla_{\bs u} \Phi(\bs u; \mathcal{B})$; see Algorithm~\ref{alg:bisection-choice-prob} below. This method has its roots in semi-parametric discrete choice theory, which exploits the structure of the marginal ambiguity set~$\mathcal B$ to compute the vector of optimal choice probabilities. Algorithm~\ref{alg:bisection-choice-prob} relies on the modulus of uniform continuity of the marginal distribution functions $F_k$, $k\in[K]$, with respect to a prescribed tolerance~$\varepsilon\geq 0$, which is defined as 
\[
    \delta(\varepsilon)=\min _{k\in [K]} \max_{\delta>0}\left\{\delta:|F_k(t_1)-F_k(t_2)| \leq \varepsilon / (2\sqrt{K}) \ \forall t_1, t_2 \in \mathbb{R} \text{ with }|t_1-t_2| \leq \delta\right\}.
\]

\begin{figure}[h!]
    \centering
    \vspace{-5em}
\begin{minipage}[t]{0.5\textwidth}
    \vspace{-18.2em}
\begin{algorithm}[H]
\caption{Bisection method for approximating the arm-sampling distribution $\bs p= \nabla_{\bs u} \Phi(\bs u; \mathcal{B})$}
\begin{algorithmic}
\REQUIRE error tolerance $\varepsilon$, reward estimate $\bs u$, marginal distribution functions $F_k$, $k\in[K]$ 
\STATE Set $\bar{\tau} \gets \max _{k\in[K]}\{- u_k-F_k^{-1}(1-1/K)\}$
\STATE Set $\underline{\tau} \gets \min _{k\in[K]}\{- u_k-F_k^{-1}(1-1/K)\}$
\FOR{$i=1,2, \ldots,\left\lceil\log _2((\bar{\tau}-\underline{\tau}) / \delta(\varepsilon))\right\rceil$}
    \STATE Set $\tau \leftarrow(\tau+\underline{\tau}) / 2$
    \STATE Set $\hat p_k \leftarrow 1-F_k\left(- u_k-\tau\right)$ for $k\in[K]$
    \STATE \textbf{if} $\sum_{k \in [K]} \hat p_k > 1$ \textbf{then} $\bar{\tau} \gets \tau$ \textbf{else} $\underline{\tau} \gets \tau$ 
\ENDFOR\\
\RETURN $\hat{\bs p}$ with $\hat p_k=(1+\sum_{\ell=1}^K F_\ell(-u_\ell-\underline{\tau}))/K -F_k\left(- u_k-\underline{\tau}\right)$ for all $k\in[K]$
\end{algorithmic}
\label{alg:bisection-choice-prob}
\end{algorithm} 
\end{minipage}\hfill 
    \begin{minipage}[t]{0.47\textwidth}
    % \vspace{-0em}
    \centering
    \includegraphics[width=0.8\columnwidth]{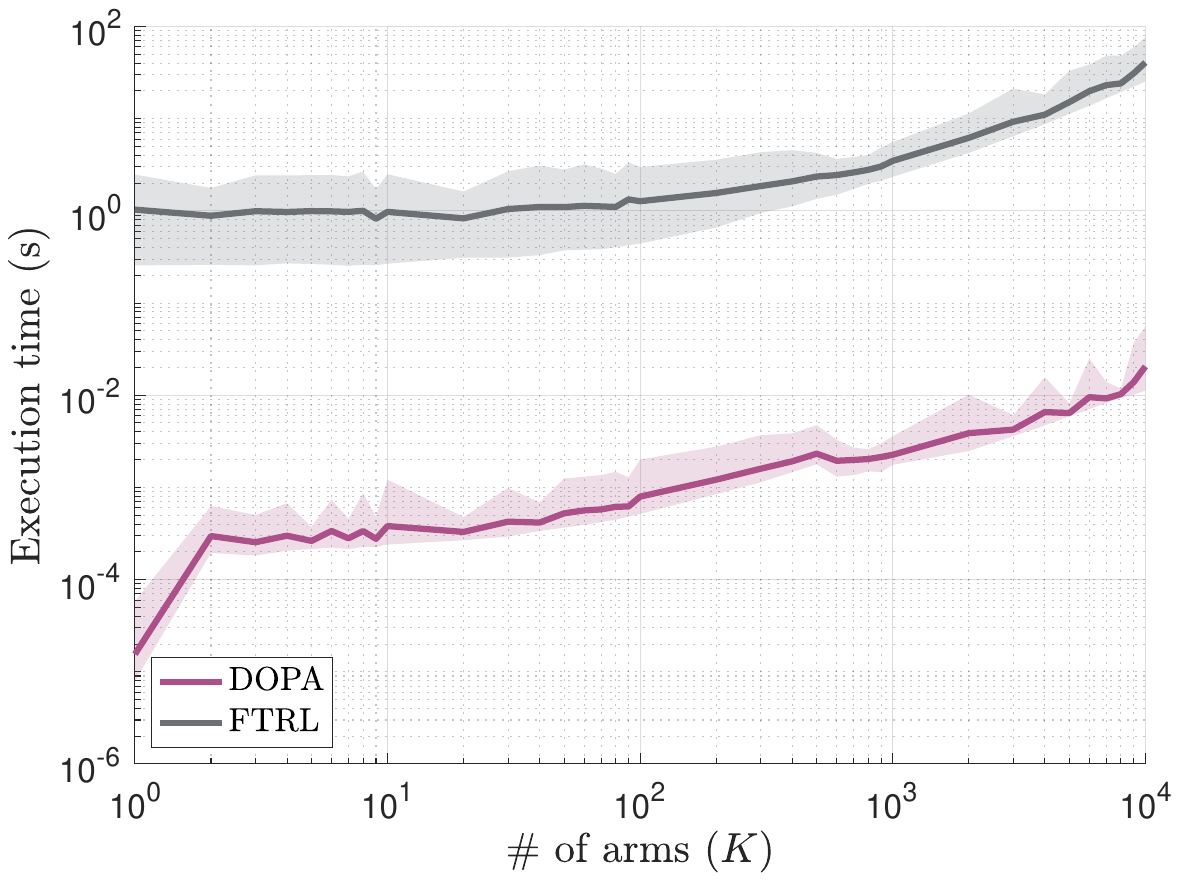}
    \vspace{-6em}
    \caption{Per-iteration runtime of DOPA (purple) and FTRL (gray) as a function of the number~$K$ of arms. The solid lines show the means, and the shaded areas visualize the corridor between the minima and maxima across 10 independent simulations ($\bs u$ is sampled uniformly from~$[0,1]^K$).}  
    \label{fig:exec_time}
    \end{minipage}
\end{figure}

The following corollary of \cite[Theorem~4.9]{ref:taskesen2023semi} characterizes the convergence behavior of Algorithm~\ref{alg:bisection-choice-prob}.

\begin{theorem}[Convergence of Algorithm~\ref{alg:bisection-choice-prob}] \label{thm:bisection-converge}
Suppose that $\mathcal{B}$ is a marginal ambiguity set of the form~\eqref{eq:marginal:set} and that the distribution functions $F_k$, $k \in[K]$ are continuous and strictly increasing in~$s$ whenever $F_k(s)\in(0,1)$. Then, for any $\bs u \in \mathbb{R}^K$ and $\varepsilon>0$, Algorithm~\ref{alg:bisection-choice-prob} outputs $\hat{\bs p}\in\Delta^K$ with $\|\hat{\bs p} - \nabla_{\bs u} \Phi(\bs u, \mathcal{B})\|_2 \leq \varepsilon$. If~$\mac B$ is additionally a Fr\'echet ambiguity set with $\bs \eta = \eta \mathbf{1}$ for some $\eta>0$ and if the marginal generator~$F$ is $L$-Lipschitz continuous whenever $F(s)\in[0,1]$, then Algorithm~\ref{alg:bisection-choice-prob} terminates after at most $\log_2(\varepsilon^{-1}2L\sqrt{K}(\bar{u} - \underline{u})/\eta)$ iterations with $\bar{u}=\max_{k\in[K]}u_k$ and $\underline{u}=\min_{k\in[K]} u_k$.
\end{theorem}

\begin{proof}
For simplicity of notation, we define $\bs p= \nabla_{\bs u} \Phi(\bs u; \mathcal{B})$. Note that the output of Algorithm~\ref{alg:bisection-choice-prob} can be expressed as $ \hat{\bs p} = \bs q + d\mathbf{1}/K  $, where $q_k=1-F_k(- u_k-\underline{\tau})$ for all $k\in[K]$ and $d = 1 - \sum_{k=1}^K q_k$ is a nonnegative normalization constant. By \cite[Theorem~4.9]{ref:taskesen2023semi}, we know that $|q_k- p_k| \le \varepsilon/(2\sqrt{K})$ for all $k\in[K]$. This implies that
\[
    \|\hat{\bs p} - \bs p\|_2 = \left\|\bs q + \frac{d\mathbf{1}}{K}- \bs p\right\|_2 \le \|\bs q - \bs p\|_2 + \left\| \frac{d\mathbf{1}}{K}\right\|_2 = \frac{\varepsilon}{2}+\frac{\varepsilon}{2}=\varepsilon,
\]
where the second equality holds because $d=1 - \sum_{k=1}^K q_k \le 1 - \sum_{k=1}^K p_k + K\varepsilon/(2\sqrt{K})=\varepsilon\sqrt{K}/2.$
As for the second claim, note that the $L$-Lipschitz continuity of~$F$ implies via the definition of~$F_k$ in~\eqref{eq:marginal_dists} that the uniform continuity parameter $\delta(\varepsilon)$ is bounded below by $\varepsilon\eta/(2L\sqrt{K})$. Also, as all components of~$\bs \eta$ are identical, one readily verifies that $\overline \tau-\underline\tau\leq \overline u-\underline u$. The number of iterations of Algorithm~\ref{alg:bisection-choice-prob} is therefore bounded above by $\log_2(\varepsilon^{-1}2L\sqrt{K}(\bar{u} - \underline{u})/\eta)$.
\end{proof}

Note that DOPA calls Algorithm~\ref{alg:bisection-choice-prob} with input $\bs u=\hat{\bs u}_{t-1}$ in each iteration~$t=1,\ldots,T$ of Algorithm~\ref{alg:mab-gbpa} in order to compute an arm-sampling distribution~$\bs p_t$. As $\hat{\bs u}_{t-1} = \sum_{s=1}^{t-1} \hat{\bs r}_{s}$, the range $\bar{u} - \underline{u}$ of the reward estimates is uncertain and depends on~$t$. In addition, as $\bs r_{s}\in [-1,0]^K$ and $\hat{\bs r}_s = (r_{s, a_s} / p_{s, a_s}) \bs e_{a_s}$ for all~$s=1,\ldots, t-1$, we have $\hat{\bs u}_{t-1}=\mathcal O(t)$ with high probability. 
%$ \|\hat{\bs u}_{t-1}\|_\infty \le t/p_{\rm{min}} $ where $p_{\rm{min}}=\min_{t \le T} p_{t,a_t}>0$. 
This observation implies that the $t$-th call of Algorithm~\ref{alg:bisection-choice-prob} has $\mathcal O(\log(\sqrt{K}t/\eta))$ iterations with high probability. In addition, each iteration runs in time $\mathcal O(K)$. Hence, if~$\eta=\mathcal O(\sqrt{T})$ (which leads to optimal regret guarantees as explained in Corollary~\ref{cor:gbpa-tsallis}), then the $t$-th call of Algorithm~\ref{alg:bisection-choice-prob} runs in time at most $\mathcal O(K\log(\sqrt{KT}))=\tilde{\mac O}(K)$ with high probability, where~$\tilde{\mac O}$ hides logarithmic factors. The efficiency of Algorithm~\ref{alg:bisection-choice-prob} used by DOPA is thus comparable to the sampling procedure used by FTPL.
%In addition, because $\hat r_{\tau,a_\tau}$ constitutes an unbiased estimator of $\EE[\hat{\bs r}_{\tau,a_\tau}],$ $\|\hat{\bs u}_t\|_\infty \propto t$ for large~$t$. 
We highlight that the marginal generators of all Fr\'echet ambiguity sets that were examined in Section~\ref{sec:optimal} and lead to optimal regret bounds satisfy the Lipschitz continuity condition of Theorem~\ref{thm:bisection-converge}.

We now compare the per-iteration complexities of DOPA and FTRL, that is, we measure the times both methods spend on computing the arm-sampling distributions. All experiments are run on a computer with an Apple M1 Pro processor with 16GB RAM, and all optimization problems are modeled in MATLAB using the YALMIP interface \citep{mccormick1976computability}. The code for reproducing Figure~\ref{fig:exec_time} is available from \url{https://anonymous.4open.science/r/bandit-experiments-FB73/}.

As for DOPA, we set $\mathcal{B}$ to a Fr\'echet ambiguity set in the sense of Definition~\ref{def:frechet-ambiguity-set} with marginal generator $F(s) = (2-s)^{-2}$ and $\bs \eta = \bs 1$. As for FTRL, we set~$\psi$ to the Tsallis entropy with parameter~$\alpha=\frac{1}{2}$. Theorem~\ref{thm:frechet-regularization} and Corollary~\ref{cor:gbpa-tsallis} then imply that $\Phi( \bs u; \mathcal{B}) = \Phi^R(\bs u
; \psi)$; see also the proof of Theorem~\ref{theorem:marginal-anytime} for further details. Hence, DOPA and FTRL use the same arm-sampling distributions and are thus equivalent. We compute the arm-sampling distribution $\nabla_{\bs u} \Phi^R(\bs u ; \psi) = \argmax_{\bs p \in \Delta^K} \bs p^\top \bs u - \psi(\bs p)$ of FTRL by solving the underlying second-order-cone program with MOSEK \citep{aps2019mosek}. In addition, we use Algorithm~\ref{alg:bisection-choice-prob} to compute the arm-sampling distribution $\nabla_{\bs u} \Phi(\bs u ; \mathcal{B})$ of DOPA to within an error tolerance of $\varepsilon = 10^{-8}$, which matches MOSEK's suboptimality tolerance for conic programs. Figure~\ref{fig:exec_time} visualizes the per-iteration runtimes of DOPA and FTRL as a function of the number~$K$ of arms. We observe that DOPA runs almost $10^4$ times faster uniformly across all~$K$.

% This result indicates that 
% \notebt{Algorithm and figure next to eachother}

\section{Concluding Remarks and Limitations}
We introduce DOPA as a new GBPA algorithm that builds a bridge between FTPL and FTRL methods. DOPA is based on an ``{\em optimism in the face of ambiguity}" principle and implicitly solves optimization problems over marginal ambiguity sets in order to determine FTPL-type noise distributions. DOPA enables us to establish a one-to-one correspondence between FTRL algorithms with additively separable regularization functions and FTPL algorithms. As a result, it circumvents the challenges associated with the regret analysis of FTPL-type algorithms and with the computational complexity of FTRL-type algorithms. Indeed, DOPA provides a unified regret analysis for perturbation-based methods by connecting them to FTRL methods, thus paving the way for new FTPL algorithms with optimal regret guarantees. In addition, the arm-sampling distributions of DOPA can be computed highly efficient with a bisection algorithm inspired by modern discrete choice theory. We show that the per-iteration complexity of DOPA exceeds that of FTPL algorithms only by logarithimc factors in~$K$ and~$T$. We see potential in exploring variants of DOPA with new Fr\'echet ambiguity sets that induce unconventional regularizers (see, {\em e.g.}, \citep[Example~3.11]{ref:taskesen2023semi}) or with completely different classes of ambiguity sets.

% {\color{red}[ML: the following paragraph is mainly about different settings for which FTRL applies (and thus DOPA applies), showing the generalizability of DOPA.]}

The design principle behind DOPA extends beyond the $K$-armed bandit setting while preserving BOBW capability. Notable future applications of DOPA include decoupled exploitation-exploration \citep{jin2024improved}, where a learner can choose to receive a reward from one arm while simultaneously gathering information about the reward from another. This concept has significant implications for the development of efficient reinforcement learning algorithms \citep{huang2022tiered}. Another potential application of our algorithm, where it can achieve a BOBW regret bound, is the dueling bandit problem \citep{zimmert2021tsallis}. In this setting, the learner selects two arms in each round to ``duel'' and receives feedback on the arm with the higher reward. Dueling bandit models have practical applications, such as hyperparameter tuning \citep{kumagai2017regret}. Furthermore, our framework generalizes the hybrid Tsallis entropy regularizers used in an FTRL-type algorithm with BOBW capability \citep{ito2024adaptive}, making it applicable to both $K$-armed bandit and linear bandit problems.

We also recognize several limitations of our work. First, certain types of regularizers cannot be captured by marginal ambiguity sets  of the form~\eqref{eq:marginal:set}. A notable example is the log-barrier regularizer considered by~\citet{jin2024improved}.
\cite[Proposition~2.2]{hofbauer2002global} shows that it is \textit{impossible} to recover an FTRL algorithm with a log-barrier regularizer using any FTPL algorithm with a stochastic perturbation whose distribution is independent of the reward estimates~$\bs u$. Second, the bisection method in Algorithm~\ref{alg:bisection-choice-prob} is efficient as long as the marginal cumulative distribution functions $F_k$ and their inverses $F_k^{-1}$ can be computed efficiently. However, for hybrid regularizers, computing \( F_k \) can be cumbersome, making the bisection method computationally inefficient for certain choices of the marginal generators~$G_1$ and~$G_2$.

\section*{Appendix: Strongly Convex Regularization Functions}
The following lemma borrowed from \citet[Proposition~4.8]{ref:taskesen2023semi} identifies sufficient conditions on the distribution functions~$F_k$, $k\in[K]$, under which the regularization function $\psi(\bs p)=- \sum_{k=1}^K \int_{1-p_k}^1 F_k^{-1}(t)\,\mathrm{d} t$ is strongly convex. It exploits a natural duality relation between smoothness and strong convexity properties. We sketch the proof of this result for completeness.

\begin{lemma}\label{lem:str:convex:regularizer}
If $\mathcal{B}$ is a marginal ambiguity set of the form~\eqref{eq:marginal:set}, and if the cumulative distribution functions $F_k, k \in[K]$, are Lipschitz continuous with Lipschitz constant~$L$, then the regularization function~$\psi(\bs p)=- \sum_{k=1}^K \int_{1-p_k}^1 F_k^{-1}(t)\,\mathrm{d} t  $ is $L$-strongly convex on $[0,1]^K$.
\end{lemma}

\begin{proof} %[Proof of Corollary~\ref{lem:str:convex:regularizer}]
The claim holds if we can show that $\psi(\bs p)-\|\bs p\|_2^2 /(2 L)$ is convex in~$\bs p$. As~$F_k$ is non-decreasing and Lipschitz continuous by assumption, we have 
\begin{align*}
L &\geq \sup _{\substack{s_1, s_2 \in \mathbb{R} \\ s_1 > s_2}} \frac{F_k\left(s_1\right)-F_k\left(s_2\right)}{s_1-s_2} 
\geq \sup _{\substack{p_k, q_k \in(0,1) \\ p_k>q_k}} \frac{(1-q_k)-(1-p_k)}{F_k^{-1}(1-q_k)-F_k^{-1}(1-p_k)},
\end{align*}
where the second inequality follows from restricting $s_1$ and $s_2$ to the image of $(0,1)$ under the (left) quantile function $F^{-1}_k$. Rearranging terms in the above inequality then yields 
\[
    -F_k^{-1}(1-q_k)-q_k/L \leq-F_k^{-1}(1-p_k)-p_k/L\quad \forall p_k, q_k \in(0,1) \text{ with } q_k<p_k.
\]
Thus, the function $-F_k^{-1}\left(1-p_k\right)-p_k / L$ is non-decreasing in $p_k$ on the open interval $(0,1)$, and its primitive $-\int_{1-p_k}^1 F_k^{-1}(t) \mathrm{d} t-p_k^2 /(2 L)$ is convex and continuous in $p_k$ on the closed interval $[0,1]$. The claim then follows because convexity is preserved under summation.
\end{proof}

\bibliographystyle{abbrvnat}
\bibliography{ref}

\end{document}

%% file: style.tex
\usepackage{algorithm}
\usepackage{algorithmic}
\usepackage{hyperref}
\usepackage[square]{natbib}
\setcitestyle{number}

\hypersetup{colorlinks, breaklinks,
linkcolor={[RGB]{255,105,180}},
citecolor={[rgb]{0.7, 0.5, 0.9}},
urlcolor={green!50!black}}

\usepackage{dsfont} 
\usepackage{bbm} % for bbm numbers
\usepackage{footnote}

\usepackage{comment}
\usepackage{enumerate}
\usepackage{multicol}

\usepackage{tcolorbox}
\usepackage{amsthm}
\usepackage{amsmath}
\usepackage{amssymb}
\usepackage{algorithm}
\newtheorem{theorem}{Theorem}[section]
\newtheorem{lemma}[theorem]{Lemma}
\newtheorem{proposition}[theorem]{Proposition}
\newtheorem{remark}{Remark}
\newtheorem{definition}{Definition}
\newtheorem{corollary}[theorem]{Corollary}

\DeclareMathOperator*{\argmax}{argmax}

\newcommand{\R}{\mathbb{R}}
\newcommand{\mac}{\mathcal}

\newcommand{\diff}{\mathrm{d}}
\newcommand{\opt}{^\star}

\newcommand{\mb}{\mathbb}
\newcommand{\EE}{\mathbb{E}}
\newcommand{\PP}{\mathbb{P}}
\newcommand{\QQ}{\mathbb{Q}}

\usepackage[font=footnotesize]{caption}
\newcommand{\bs}{\boldsymbol}
  
\usepackage{enumerate}
\usepackage{enumitem}

\graphicspath{{figures/}}
\usepackage{xcolor}

\allowdisplaybreaks